\documentclass[english,twocolumn,transaction,10pt]{IEEEtran}

\usepackage{comment}
\usepackage[T1]{fontenc}
\usepackage{float}
\usepackage{amsmath}
\usepackage{amsthm}
\usepackage{amssymb}
\usepackage{stackrel}
\usepackage{graphicx}
\usepackage{subcaption} 
\usepackage{setspace}
\usepackage{bm}
\usepackage{caption}
\usepackage{xcolor}
\usepackage{cases}
\usepackage{algorithm}
\usepackage{algorithmic}
\usepackage{stfloats}

\makeatletter

\floatstyle{ruled}
\newfloat{algorithm}{tbp}{loa}
\providecommand{\algorithmname}{Algorithm}
\floatname{algorithm}{\protect\algorithmname}

\theoremstyle{plain}

\theoremstyle{definition}

\theoremstyle{plain}

\theoremstyle{definition}

\theoremstyle{plain}

\newtheorem{define}{Definition}
\newtheorem{assumption}{Assumption}

\newtheorem{theo}{Theorem}
\newtheorem{remark}{Remark}

\IEEEoverridecommandlockouts
\usepackage{ifpdf}
\usepackage{cite}
\ifCLASSOPTIONcompsoc
\usepackage[caption=false,font=normalsize,labelfont=sf,textfont=sf]{subfig}
\else
\usepackage[caption=false,font=footnotesize]{subfig}
\fi
\usepackage{amsmath}
\usepackage{mathtools}
\usepackage{booktabs}
\usepackage{multirow}
\usepackage{setspace}
\usepackage{lettrine}
\usepackage{setspace}
\usepackage{array}
\@ifundefined{showcaptionsetup}{}{
	\PassOptionsToPackage{caption=false}{subfig}}
\usepackage{subfig}
\makeatother
\allowdisplaybreaks[4]
\usepackage{diagbox}
\usepackage{babel}

\usepackage{xspace}
\def\eg{\textit{e.g.,}\xspace}
\def\etc{\textit{etc}\xspace}
\def\ie{\textit{i.e.,}\xspace}
\def\etal{\textit{et al.}\xspace}

\newtheorem{lemma}{Lemma}

\providecommand{\keywords}[1]{\textbf{\textit{Index terms---}} #1}

\begin{document}
	\captionsetup[figure]{font={small}, name={Fig.}, labelsep=period}
	
	\title{ Free Privacy Protection for Wireless Federated Learning: Enjoy It or Suffer from It?}

		\author{Weicai~Li,~\textit{Graduate Student Member, IEEE},~Tiejun~Lv,~\textit{Senior Member, IEEE},\\~Xiyu~Zhao,~Xin~Yuan,~and~Wei~Ni\\ 

\thanks{Manuscript received 24 July 2024; revised 06 February 2025, 03 May 2025, and 11 June 2025; accepted 12 June 2025. This paper was supported in part by the National Natural Science Foundation of China under No. 62271068, and the Beijing Natural Science Foundation under Grant No. L222046. (\emph{corresponding author: Tiejun Lv}.)}

\thanks{W. Li, T. Lv, and X. Zhao are with the School of Information and Communication Engineering, Beijing University of Posts and Telecommunications, Beijing 100876, China. W. Li is also with the University of Technology Sydney, Broadway, NSW, Sydney. X. Zhao is also with Macquarie University, Macquarie Park, NSW, Sydney. (e-mail: \{liweicai, lvtiejun, zxy\}@bupt.edu.cn). 

X. Yuan is with Tongji University, Shanghai 201804,
China.

W. Ni is with Macquarie University, Macquarie Park, NSW 2109, Australia.}}
	
\maketitle

\begin{abstract}
Inherent communication noises have the potential to preserve privacy for wireless federated learning (WFL) but have been overlooked in digital communication systems predominantly using floating-point number standards, \eg IEEE~754, for data storage and transmission. This is due to the potentially catastrophic consequences of bit errors in floating-point numbers, \eg on the sign or exponent bits. This paper presents a novel channel-native bit-flipping differential privacy (DP) mechanism tailored for WFL, where transmit bits are randomly flipped and communication noises are leveraged, to collectively preserve the privacy of WFL in digital communication systems.  
The key idea is to interpret the bit perturbation at the transmitter and bit errors caused by communication noises as a bit-flipping DP process. 
This is achieved by designing a new floating-point-to-fixed-point conversion method that only transmits the bits in the fraction part of model parameters, hence eliminating the need for transmitting the sign and exponent bits and preventing the catastrophic consequence of bit errors.
We analyze a new 
metric to measure the bit-level distance of the model parameters and prove that the proposed mechanism satisfies $(\lambda,\epsilon)$-Rényi DP and does not violate the WFL convergence. 
Experiments validate privacy and convergence analysis of the proposed mechanism and demonstrate its superiority to the state-of-the-art Gaussian mechanisms that are channel-agnostic and add Gaussian noise for privacy protection. 
\end{abstract}
\section{Introduction}
\label{sec:intro}
Privacy-preserving federated learning (FL) integrates privacy models into a distributed machine learning (ML) framework, offering provable privacy assurances~\cite{Tang2023PILE,Zhou2023A,bassily2014private, DLDP, Chollet2017Xception}.
Differential privacy (DP) is prevalent with several variants studied extensively, including $\epsilon$-DP~\cite{barthe2011information}, $(\epsilon, \delta)$-DP~\cite{acs2018differentially}, Rényi DP~\cite{Mironov_2017}, and $(\lambda, \epsilon)$-Rényi DP~\cite{zhu2017differentially,fu2021practicality, Mironov_2017}.
Combining DP with FL permits clients to train their local models within specified privacy protection levels
\cite{Dong2023Privacy,Zhao2020Local}.
Existing privacy-preserving mechanisms, \eg Gaussian mechanism~\cite{9252950,Wei2020Federated,10073536,fu2021practicality} and  Laplace mechanism~\cite{10073536,fu2021practicality},
typically add artificial DP noise to the (decimal) local model parameters at the clients' side before digitizing and uploading them to a server.

\textbf{Motivation.} 
Wireless channels are noisy and prone to errors due to various factors like thermal noise—also known as Johnson-Nyquist noise—which results from the random movement of electrons in electronic circuits and cannot be entirely eliminated~\cite{hansen1986nyquist}. 
Nevertheless, the communication noises and their resulting errors may offer an unexpected benefit for privacy protection in wireless FL~(WFL). 

Both the benefits of leveraging communication noises for preserving the privacy of WFL and the consequences of not doing so are obvious.
On the one hand, a limited amount of communication errors in the received local models can be tolerated to enhance privacy preservation, as shown in over-the-air (OTA) FL~\cite{9252950,10364879}. By tolerating such errors, more local models with tolerable, mild errors can be aggregated within the available time and spectrum resources, compared to conventional methods that discard erroneous packets or wait for retransmissions~\cite{9435350,Wei2020Federated}.
On the other hand, if the privacy-preserving potential of noisy wireless channels is overlooked, unnecessarily tight privacy budgets and high DP noises would be required, thus hindering the convergence of WFL. 
In this sense, WFL either enjoys the inherent privacy protection offered by noisy wireless channels, or suffers from ineffective communication resource utilization and penalized convergence and accuracy.

No existing studies have exploited the inherent privacy protection offered by noisy wireless channels in digital communication systems.
The authors of~\cite{9252950,9693141,10364879} proposed a privacy-preserving over-the-air (OTA) FL system that leverages communication noise to enhance privacy under analog communication settings. However, implementing such an analog approach would face practical challenges, including high noise sensitivity and low spectrum efficiency, which have driven the global transition from analog communication to digital transmission (e.g., cellular networks, broadcasting, etc.).
The authors of \cite{lai2022bitaware} introduced a bit-aware randomized response mechanism for local differential privacy. Considering that different bits have different impacts on embedded features, the mechanism heuristically assigns the bits with more significance with smaller randomization probabilities. 
However, communication noise is generally assumed to impact all transmitted bits uniformly, meaning each bit has an equal probability of being flipped.

\textbf{Research Challenges.} 
It is non-trivial to exploit the communication noises 
for the privacy of WFL in digital communication systems due to the predominant use of floating-point number standards, \eg IEEE~754~\cite{8766229}.
In practice, data and information are stored and transmitted in bits following the floating-point number standards, which necessitate fidelity during data storage and transmission. 
For example, the IEEE~754 floating-point standard~\cite{8766229} encodes every number into bits with varying levels of significance, including a sign bit, exponent bits, and fraction bits. 
A single error in the sign or exponent bits can drastically distort the value.
To this end, WFL, when adhering to current practical floating-point standards as implicitly assumed in existing FL designs~\cite{Wei2020Federated,10073536}, neither tolerates communication errors nor facilitates the exploitation of communication noise for privacy protection.
Moreover, it is challenging to determine the sensitivity and subsequently the privacy budget of the bitstreams, \eg IEEE~754 floating-point numbers. The reason is that two bitstreams can represent similar (decimal) values but differ dramatically in bits. Exploiting communication noises for privacy protection in WFL requires careful consideration of these factors to balance error tolerance with meaningful privacy benefits.

\textbf{Contribution.} This paper presents a novel \textit{channel-native bit-flipping DP mechanism} tailored for WFL in digital communication systems using floating-point standards. 
By integrating both artificial and communication noises at the bit level, this mechanism ensures compliance with $(\lambda,\epsilon)$-Rényi DP, thereby bolstering the privacy of WFL. We provide a comprehensive analysis of the statistical properties of bit-flipping noise and its impact on the convergence of WFL.

The proposed mechanism offers privacy-by-design WFL through its compatibility with digital systems and by exploiting the inherent privacy-preserving properties of noisy wireless channels. 
The mechanism transforms the model parameters of WFL from standard floating-point representations, necessitating error-free transmissions, into a new fixed-point format that only requires the fraction bits to be transmitted, tolerates bit errors in the model parameters received, and can thus ``enjoy'' communication noise for privacy protection. 

The contributions of this paper are summarized as follows.
\begin{itemize}
    \item 
    We propose a novel mechanism to protect the privacy of WFL by leveraging inherently noisy wireless channels. 
    The mechanism obfuscates model parameters in bits, compatible with floating-point standards, \eg  IEEE~754, widely adopted in the systems.
    
    \item 
  We design a novel floating-point-to-fixed-point conversion that not only significantly reduces the number of transmitted bits but, more importantly, eliminates the risk of catastrophic model distortion associated with transmissions and noise corruption of the sign or exponent bits of floating-point model parameters.

    \item 
    We prove that the proposed mechanism strictly satisfies $(\lambda,\epsilon)$-Rényi DP. This involves the analysis of a new metric to measure the bit-level distance, concerning bitstreams encoded from local model parameters, and rigorous derivation of the bit-flipping probability.

\item 
We analytically corroborate the convergence of WFL under the proposed mechanism, revealing that the convergence upper bound is affected not only by the privacy budget but also by the number of total training iterations and local epochs per communication round. 
\end{itemize}

Extensive experiments validate the privacy and convergence analysis of the proposed channel-native bit-flipping DP mechanism for WFL.
Two image classification tasks are performed, using a convolutional neural network (CNN) on the Federated Extended MNIST (F-MNIST) dataset and using the 18-layer residual network (ResNet-18) on the Federated CIFAR100 (Fed-CIFAR100) dataset. Our mechanism  markedly outperforms its channel-agnostic counterparts, and 
exhibits substantial superiority to the state-of-the-art Gaussian mechanism developed under the prerequisite that only the error-free parts of (obfuscated) local model parameters are aggregated.

The rest of this paper is organized as follows. The related works are discussed in Section~\ref{sec:relwork}, followed by the system model in Section~\ref{sec:pre}. In Section~\ref{model}, we present the new channel-native bit-flipping DP mechanism. The privacy and convergence of WFL under the mechanism are discussed in Sections~\ref{sec:privacy} and~\ref{sec:convergence}, respectively. Experiments are presented in Section~\ref{section:results}, followed by conclusions in Section~\ref{sec:conclusion}. 
Table \ref{notations} collates the notation used in this paper.

\begin{table}[t]
   \small
		\centering
		\caption{Notation and definitions}\label{notations}
		\begin{tabular}{ l|p{7cm}  }
			\hline
			\textbf{Notation}& \textbf{Definition} \\
			\hline
			$n,\,N$   &Device index, and total number of devices \\
            $\mathcal{N}$ & Set of device indices \\
			$m,\,M$   & Dimension index, and total number of model dimensions \\
            
            $j$ & Bit index in the fraction part\\
            $K$   & Total number of communication rounds\\
            $E$   & Total number of local training iterations per round\\
            $\mathcal{I}_{E}$ & Index set of iterations that collects global aggregations\\
            $t,\,T$   & Index and total number of local training iterations\\
            $d,\mathcal{D}$   &Data point, and dataset\\
            
           $\mathbf{w}_{t,n}$ & Input  model of device $n$ at the $t$-th iteration\\   
            $\boldsymbol{\omega}_{t,n}$ & Output model of device $n$ at the $t$-th iteration\\  
           $\tilde{\boldsymbol{\omega}}$&Recovered local model (with error) at server\\ 
           $\tilde{\boldsymbol{u}}$&Bitstream perturbed by artificial noise\\ 
           $\hat{\boldsymbol{u}}$&Bitstream perturbed by artificial and communication noise\\
              $\boldsymbol{\omega}_{t,\mathcal{G}}$&Global model at the $t$-th local training iteration\\
             $\Delta\boldsymbol{\omega}_{\max} $& Classical sensitivity\\
             ${{\bar\kappa}_{\Delta\boldsymbol{\omega}_{\max} }}$&Expected bit-level distance\\
             $\nu_2,\,\nu_{\infty}$ &  Maxima of the $\ell_2$-norm and $\ell_{\infty}$-norm of the model parameter vector\\

		\hline
		\end{tabular}
	\end{table}

\section{Related Work}\label{sec:relwork}

This section summarizes the studies on privacy preservation in FL/WFL systems. Although many have been on the privacy of WFL, none have perturbed WFL model parameters at the bit level nor considered practical floating-point data formats used in digital systems.

\begin{figure}[t]
    \centering
\includegraphics[scale=0.6]{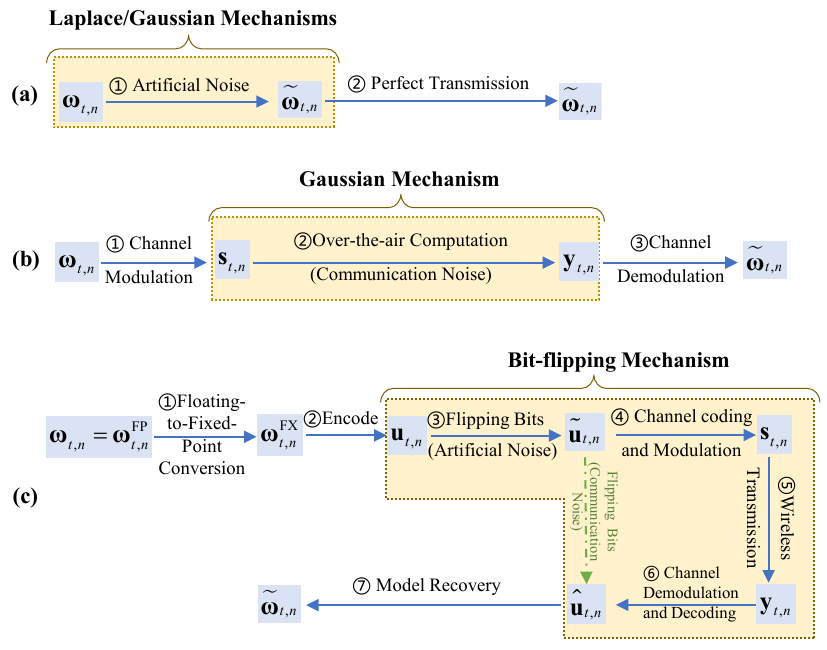}
    \caption{The workflows of the local model transmission of device $n$ under different DP mechanisms when $t\in\mathcal{I}_E$: (a) Channel-agnostic DP mechanisms, \eg Laplace or Gaussian mechanism, only use artificial noise and do not exploit noisy communication channels; (b) The DP mechanism for OTA-FL systems uses the additive Gaussian noise to protect privacy, but does not support digital transmission; (c) The proposed channel-native bit-flipping DP mechanism. }
    \label{fig:workflow}
\end{figure}
\subsection{Privacy-Preserving FL/WFL}
Existing studies have mainly focused on perturbing the model parameters in decimal amplitude prior to digitization to bitstreams, as depicted in Fig.~\ref{fig:workflow}(a). 
In~\cite{Wei2020Federated} and~\cite{10073536}, artificial Gaussian noise was applied to the (decimal) local model parameters before model transmission and aggregation, where the DP requirement was satisfied by adapting the variance of the Gaussian noise.
The authors of~\cite{fu2021practicality} and~\cite{10073536} demonstrated analytically and experimentally that the DP-based FL/WFL with the Laplace mechanism is worse than that with the Gaussian mechanism because the Laplacian mechanism adds larger noises sampled from a Laplacian distribution, which is farther away from its mean.
Yuan \etal~\cite{10073536} designed a novel amplitude-varying perturbation mechanism, where the DP noise is a geometric series changing over global aggregations to balance the privacy and utility in FL/WFL. 
All these schemes perturb the model parameters of FL/WFL in the decimal domain before the model parameters are digitized into bitstreams, yielding floating-point number structures.

\subsection{WFL under Imperfect Communications }\label{sec:wfl}
Many studies have attempted to develop communication-efficient algorithms to improve the performance of WFL, without consideration of the privacy of model parameters. Model compression~\cite{9660377,10364879} and quantization~\cite{9834296} have been exploited to reduce communication overhead, at the cost of FL performance due to compression and quantization errors. In~\cite{9435350}, it was revealed that the user transmission success probability, affected by unreliable and resource-constrained wireless channels, critically affects the convergence of WFL.
In~\cite{10032555}, a user datagram protocol with retransmissions was employed to recover from model transmission failures at the cost of delays and signaling overhead.
Operating in the analog domain, OTA-FL allows the participating clients to upload their models within the same resource block~\cite{9076343,10145450,10364879}.
These models can be aggregated over the air due to the additivity of signals in the electromagnetic field.
OTA-FL suffers from imperfect channel state information (CSI) with estimation errors~\cite{10261509}. 
No consideration has been given to the privacy of the model parameters in these studies.

\subsection{Privacy-Preserving WFL}

Some studies~\cite {10041968,9693141,9252950} have started to utilize the system-level noises, \eg quantization errors or communication noise, to implement DP mechanisms. 
None is applicable to digital communication systems that have the potential to enjoy privacy protection offered by noisy wireless channels.
In~\cite{10041968}, devices quantize and compress their local model parameters with a required bit rate, resulting in a quantization error that helps the system to satisfy a predefined privacy level. However, the privacy-preserving capability of noisy communication channels was overlooked.
In the context of OTA-FL, Liu \etal~\cite{9252950} exploited the additive Gaussian communication noise to perform the analog waveforms modulated from uncoded model parameters as a privacy-preserving mechanism, and the transmit power was dynamically controlled to satisfy the privacy requirement, as illustrated in Fig.~\ref{fig:workflow}(b).  In~\cite{9693141}, the use of Gaussian communication and artificial noises was extended to decentralized OTA-FL systems. A prescribed privacy requirement was proven in peer-to-peer communication. However, both of the studies~\cite{9252950,9693141} relied on the prerequisite of OTA-FL, only applicable to Gaussian noises in the analog signal domain, and cannot be extended to digital systems.

\section{Preliminaries and System Model}\label{sec:pre}

In this section, we provide the preliminaries of the WFL model, $(\lambda,\epsilon)$-Rényi DP, and floating-point number structure, and the considered threat model in the context of WFL.

\subsection{WFL Model}

Consider a WFL system with a central server (\eg access point or AP) and $N$ devices. $\mathcal{N}=\{1,\cdots,N\}$ collects the indexes of the devices. Device $n\in\mathcal{N}$ owns a dataset ${{\cal D}}_n$ with $|{{\cal D}}_n|$ data samples. The local loss function of device $n$ is  
 \begin{align}
    F_n(\mathbf{w},\mathcal{D}_n)=\frac{1}{|{{\cal D}}_n|}{\sum}_{d\in{\cal D}_{n}} {f}(d,\mathbf{w}),
 \end{align}
where $\mathbf{w}\in \mathbb{R}^{M\times 1}$ is an $M$-dimensional model parameter vector, $d \in {{\cal D}}_n$ is the data sample, and ${f}(d,\mathbf{w})$ is the loss function concerning $d$ and~$\mathbf{w}$. 

The objective of the WFL system is to  
find the optimal model parameter vector $\mathbf{w}^*=\underset {\mathbf{w}\in \mathbb{R}^{M\times 1}}{ \arg\, \min}\,F(\mathbf{w})$, with  $F(\mathbf{w})$ being the global loss function which is the weighted sum of the local loss functions of all devices, as given by 
 \begin{align}
      F(\mathbf{w})={\sum}_{n\in\mathcal{N}} q_nF_n(\mathbf{w},\mathcal{D}_n),\label{FW}
 \end{align} 
where $q_n=\frac{|{{\cal D}}_n|}{{\sum}_{\forall n} |{{\cal D}}_n|}$ is the coefficient of device $n$.
Let $F_n^*$ and $F^*$ denote the minima of $F_n(\mathbf{w})$ and $F(\mathbf{w})$, respectively.

The WFL training consists of training iterations, with adequate synchronization across all devices.
Let $K$ denote the total number of communication rounds for each device, and $E$ denote the number of local training iterations per communication round. The training process has a total of $T=KE$ local training iterations. Let $\mathcal{I}_{E} = \{ 0, E, 2E, \cdots, KE\}$ collect the iterations when global aggregations take place.

At the beginning of the $k$-th communication round, the server broadcasts the latest global model to all clients. Each device $n$ trains its local model for $E$ local iterations based on the global model and its local data ${\cal D}_n$. The local models of the devices, $\boldsymbol{\omega}_{t,n},\,\forall n$, are updated in parallel when $t\notin\mathcal{I}_{E}$. 
  
When $t\in\mathcal{I}_{E}$, the devices first update their local models $\boldsymbol{\omega}_{t,n},\,\forall n$, 
then randomly flips the fraction bits of its $m$-th element, \ie  $\omega_{t,n,m}\in \boldsymbol{\omega}_{t,n}$ with $ m=1,\cdots,M$, to perturb $\boldsymbol{\omega}_{t,n}$ and uploads the perturbed local models through a noisy channel to the server.
The server aggregates the received local models, \ie $\tilde{\boldsymbol{\omega}}_{t,n},\,\forall n$, to update the global model, $\boldsymbol{\omega}_{t,\mathcal{G}}=\sum_{ n\in\mathcal{N}}q_n\tilde{\boldsymbol{\omega}}_{t,n}$, and broadcasts $\boldsymbol{\omega}_{t,\mathcal{G}}$ to all devices for further updating their local models.

Suppose full-batch gradient descent over training iterations, $\forall t=1,2,\cdots,T$. 
 The local model of device~$n$ at the  $(t+1)$-th (local) iteration is updated as 
\begin{align}\label{eq:local update 1}
    &\boldsymbol{\omega}_{t+1,n}=\mathbf{w}_{t,n}-\eta\cdot{\rm Clip}_{G}\left(\nabla F_{n}\left(\mathbf{w}_{t,n},\mathcal{D}_n\right)\right), \quad \forall t,n,
\end{align}
where \begin{align}
    \label{eq:local update 2}
    &\mathbf{w}_{t,n}=\left\{ \begin{array}{ll}
\boldsymbol{\omega}_{t,n}, & \text{if }t\notin\mathcal{I}_{E};\\\boldsymbol{\omega}_{t,\mathcal{G}}=
{\sum}_{n\in\mathcal{N}}q_n\tilde{\boldsymbol{\omega}}_{t,n}, & \text{if }t\in\mathcal{I}_{E},
\end{array}\right.
\end{align} is the input of the $(t+1)$-th (local) iteration, and $\boldsymbol{\omega}_{t+1,n}$ is the corresponding output of the local training.
Moreover, $\eta$ is the learning rate, $\nabla F_n({\mathbf{\cdot}})$ denotes the  
gradient of $F_n({\mathbf{\cdot}})$, ${\rm Clip}_{G}\left(\nabla F_{n}\left(\mathbf{w},\mathcal{D}_n\right)\right)=\frac{1}{|{{\cal D}}_n|}{\sum}_{d\in{\cal D}_{n}} \nabla{f}(d,\mathbf{w})\cdot \min \{1,\frac{G}{\Vert\nabla{f}(d,\mathbf{w})\Vert_2}\}$ is the clipped operator with the gradient clipping threshold $G$\footnote{ We adopt gradient clipping to ensure stable training and compatibility with practical ML algorithms. Gradient clipping has been extensively adopted in ML (with or without DP) to mitigate gradient explosion and preserve convergence; see \cite{pmlr-v202-koloskova23a} and references therein.}. 

Next, we pre-train the model and record the maxima of the $\ell_2$-norm and $\ell_\infty$-norm of the model parameter vector, denoted by $\nu_{2}=\underset{\forall t,n}{\max}\; \Vert\boldsymbol{\omega}_{t,n}\Vert_2$ and $\nu_{\infty}=\underset{\forall t,n}{\max}\; \Vert\boldsymbol{\omega}_{t,n}\Vert_{\infty}$, respectively, both of which increase with~$G$.

\subsection{Threat Model }\label{sec: threat model}

A critical risk arises from within the WFL system due to an honest but curious server. The server can legally access the local models trained by the devices, infer their training data, and compromise their privacy; \ie launching an inference attack~\cite{9500936}. This happens when the server has overlapping datasets with the devices and can analyze their local training inputs and outputs to infer their local data. 

Encryption cannot stop the server from launching inference attacks on the local models, but it can prevent outsiders from accessing the local models.
As a matter of fact, the transmissions between the devices and the server are typically encrypted.
Computationally efficient symmetric encryption, such as Advanced Encryption Standard (AES)~\cite{rao2017survey} and Data Encryption Standard (DES)~\cite{MATTA202111035}, can be reasonably employed to secure the model transmissions.  
While homomorphic encryption can prevent the server from deciphering the local models, it requires all devices to use the same private key and cannot stop them from eavesdropping on each other. Moreover, homomorphic encryption is typically computationally expensive and needs a trusted third party for key agreement~\cite{acar2018survey}. 

To this end, it is reasonable to consider the threat imposed by an honest but curious server. For illustration convenience, we assume the uploaded local models are not encrypted when elaborating on the proposed channel-native bit-flipping DP mechanism. Then, we show that the mechanism can apply when the local models are encrypted with symmetric keys.

\subsection{$(\lambda,\epsilon)$-Rényi DP}

In this paper, we consider the Rényi DP mechanism to protect the privacy of the devices. 
The definition of $(\lambda,\epsilon)$-Rényi DP is provided as follows.

\begin{define} [\textbf{$(\lambda,\epsilon)$-Rényi DP~\cite{Mironov_2017}}]
\textit{The randomized mechanism $\mathcal{M}:\mathcal{D}\mapsto\mathbb{R}$ satisfies $\epsilon$-Rényi DP of order $\lambda$, $(\lambda,\epsilon)$-Rényi, if, for any adjacent datasets $\mathcal{D}$ and $\mathcal{D}'$, it holds that }
\begin{align}
D_{\lambda }(\mathcal{M}(\mathcal{D})\|\mathcal{M}(\mathcal{D}'))\leq \epsilon,\notag
\end{align}
\textit{where the Rényi divergence of order $\lambda>1$ is defined as}
\begin{align}
	\!\!\!\!\!\! D_{\lambda }(\mathcal{M}(\mathcal{D})\|\mathcal{M}(\mathcal{D'}))\!\!\triangleq\!\!\frac{1}{\lambda\!-\!1}\ln\!\left[\mathbb{E}_{x\sim \mathcal{M}(\mathcal{D}')}\Big(\frac{\mathcal{M}(\mathcal{D})}{\mathcal{M}(\mathcal{D}')}\Big)^{\lambda}\right]\notag.
\end{align}
\end{define}
Note that the Rényi divergence recedes to the Kullback-Leibler divergence, \ie $D_{1}(\mathcal{M}(\mathcal{D})\|\mathcal{M}(\mathcal{D}'))=D_{\mathrm{KL}}(\mathcal{M}(\mathcal{D})\|\mathcal{M}(\mathcal{D}'))$, when $\lambda=1$.

\subsection{Floating-Point Number Format}

For illustration, we consider the single-precision floating-point format\footnote{IEEE~754 also defines other binary floating-point formats, \eg half-precision, double-precision, quadruple-precision, and octuple-precision. The proposed analysis and algorithm are applicable to all these formats.}, \textbf{binary32}, which is the most commonly used computer number representation that typically requires 32 bits of memory, and follows the IEEE~754 standard~\cite{8766229}. 
A \textbf{binary32} number comprises a sign bit, 8 exponent bits, and 23 fraction bits with an additional, implicit leading ``1'' bit to the left of the binary point~\cite{8766229}.
The (decimal) value of a \textbf{binary32} number, $a=(a_{31}a_{30}\cdots a_{0})_{\rm 2}$, is\footnote{The implicit ``1'' bit before the fractional bits is not explicitly stored in the memory to save space and remains ``1'' for normalized numbers~\cite{8766229}.}
\begin{align}
       \notag a&=(-1)^{a_{31}}\times2^{(a_{30}a_{29}\cdots a_{23})_{2}-127}\times(1.{a}_{22}\cdots{a}_{0})_{2}\\\notag&=(-1)^{a_{31}}\times2^{(a_{30}a_{29}\cdots a_{23})_{2}-127} \times\Big(1+{\sum} _{i=1}^{23}a_{23-i}2^{-i}\Big).
\end{align}

When the bits in the fraction part of a \textbf{binary32} floating-point number are randomly flipped with probability $p$, the mean and expectation of the resulting floating-point number in decimal are provided in the following lemma.

\begin{lemma}\label{lemma:expectation and Var}
    For a \textbf{binary32} number $a=(a_{31}a_{30}\cdots a_{0})_{\rm 2}$, after flipping its bits with a probability of $p$ in the fraction part, the mean of the output $\tilde{a}$ in decimal is given by 
    \begin{align}\label{expectation_of_a}
      \!\!\!\!  \mathbb{E}(\tilde{a})\!\! =\! \!(1\! \! -\!\!2p)a\!\!+\!\!(\!-\!1)^{a_{31}} 2^{(a_{30}a_{29}\dots a_{23})_{2}\!-\! 127}\Big(\!2p\!+\!\! \! 
      {\sum}_{i\!=\!1}^{23}2^{\! -\! i}p\!\Big)\!,
    \end{align}
   and the variance of $\tilde{a}$ is given by 
   \begin{align}\label{Var_of_a}
       \mathrm{Var}(\tilde{a})\!\!=\frac{(1\!\!-\!\!4^{\!-\!23})}{3}\!p(1\!\!-\!\!p)\!\left(\!2^{2(a_{30}a_{29}...a_{23})_{2}\!-\!254}\right).
   \end{align}
\end{lemma}
\begin{proof}
    See \textbf{Appendix~\ref{proof_E_Var_of_a}}.
\end{proof}

\section{Channel-Native Privacy Protection for WFL}\label{model}
In this section, we delineate the new \textit{channel-native bit-flipping DP mechanism.} 
Considering numerical data is stored and transmitted in bits, \eg using the \textbf{binary32} floating-point format,
and noisy wireless channels cause bit errors, we protect the privacy of the model parameters by having their bits randomly flipped by the clients and noisy wireless channels. 

The new mechanism involves a new method for converting floating-point model parameters to fixed-point counterparts so that only the bits accounting for the fraction part of the model parameters can be meaningfully perturbed (as flipping the sign or exponent bits could dramatically distort the model parameters). The mechanism also involves a new way to evaluate the sensitivity regarding the number of different bits between fixed-point model parameters.

\subsection{Proposed Floating-Point-to-Fixed-Point Conversion}\label{subsection_map}

When $ t\in\mathcal{I}_E$, each element of $\boldsymbol{\omega}_{t,n}$ is  bounded by the $\ell_{\infty}$-norm 
 of $\boldsymbol{\omega}_{t,n}$, \ie  $|\omega_{t,n,m}|\leq\nu_{\infty}$, where $\omega_{t,n,m}$ is the $m$-th element of $\boldsymbol{\omega}_{t,n}$ with $m=1,\cdots,M$. 
The \(\ell_{\infty}\)-norm, \(\nu_{\infty}\), is expressed in the \textbf{binary32} format as \(\nu_{\infty} = (-1)^{c_{31}} \times 2^{(c_{30}c_{29}\dots c_{23})_{2} - 127} \times \left(1 + \sum_{i=1}^{23}1\cdot c_{23-i} \cdot 2^{-i}\right)\). Since \(\nu_{\infty} \geq |\omega_{t,n,m}| \geq 0\) is positive, the sign bit \(c_{31}\) is set to $0$. The upper bound of \(\nu_{\infty}\) is achieved when \(c_{22}c_{21}\dots c_{0} = 11\dots1\), resulting in \(1 + \sum_{i=1}^{23} 2^{-i} = 2 - 2^{-23} \leq 2\). By contrast, the lower bound is achieved when \(c_{22}c_{21}\dots c_{0} = 00\dots0\), yielding \(1 + \sum_{i=1}^{23} 0 \cdot 2^{-i} = 1\). Consequently, \(\nu_{\infty}\) is bounded as \(2^{(c_{30}c_{29}\dots c_{23})_{2} - 127} \leq \nu_{\infty} \leq 2^{(c_{30}c_{29}\dots c_{23})_{2} - 126}\). After clipping, the model parameter \(\boldsymbol{\omega}_{t,n}\) is restricted to the range \(\left[-2^{(c_{30}c_{29}\dots c_{23})_{2} - 126}, 2^{(c_{30}c_{29}\dots c_{23})_{2} - 126}\right]^{M \times 1}\).

A critical and novel step of the proposed approach is converting \(\boldsymbol{\omega}_{t,n}\triangleq\boldsymbol{\omega}^{\rm FP}_{t,n}\) from a floating-point number into a fixed-point number \(\boldsymbol{\omega}^{\rm FX}_{t,n}\) by adding a constant to each model element \(\omega_{t,n,m},\,\forall m\) of \(\boldsymbol{\omega}^{\rm FP}_{t,n}\). The constant is \(3 \times 2^{(c_{30}c_{29}\cdots c_{23})_{2} - 126}\), determined by the exponent part of the floating-point number \(\nu_{\infty}\), \ie \((c_{30}c_{29}\cdots c_{23})_{2}\). 
Adding this constant to every element of \(\boldsymbol{\omega}^{\rm FP}_{t,n}\) 
converts \(\boldsymbol{\omega}^{\rm FP}_{t,n}\) from a \textbf{binary32} floating-point number to a fixed-point number within the following range:
\begin{equation}
\begin{aligned} 
\boldsymbol{\omega}^{\rm FX}_{t,n}\triangleq  &\boldsymbol{\omega}^{\rm FP}_{t,n} + \{3 \times 2^{(c_{30}c_{29}...c_{23})_{2} - 126}\}^{M\times 1} \\
&\in [2^{(c_{30}c_{29}...c_{23})_{2} - 125}, 2^{(c_{30}c_{29}...c_{23})_{2} - 124}]^{M\times 1}.
\end{aligned}\label{fp2fx}
\end{equation}
By this means, all elements of $\boldsymbol{\omega}^{\rm FX}_{t,n},\forall n,t$ share the same positive ($+$) sign and exponent values of $(c_{30}c_{29}\cdots c_{23})_2+2$.

The proposed floating-point-to-fixed-point conversion eliminates the need to transmit the sign and exponent bits of the original floating-point number, \ie $c_{31}c_{30}\cdots c_{23}$, since those bits are converted to $(c_{30}c_{29}\cdots c_{23})_2+2$ with $c_{31}=0$, which are consistent across all devices and known to the server (as $c_{30}c_{29}...c_{23}$ depends only on the common $\nu_{\infty}$). 
To obtain the public $\nu_{\infty}$, the server can pre-train the FL model using publicly available datasets, compute the $\ell_{\infty}$-norm (denoted as $\nu_{\infty}$), and then broadcast this public information to all clients.
Each client can also pre-train its model locally and compute its own $\ell_{\infty}$-norm. If the public $\ell_{\infty}$-norm is larger than the client's local $\ell_{\infty}$-norm, the client decides to participate in the FL task. In other words, $\nu_{\infty}$ is publicly shared and does not leak privacy.
More importantly, withholding the sign and exponent bits prevents the sign and exponent bits of the floating-point model parameters from being flipped at the devices and during transmissions through noisy wireless channels, as flipping those bits could drastically distort the model parameters.

As a result of the floating-point-to-fixed-point conversion, each device $n$ only needs to encode the 
fraction part of the fixed-point model parameter {\(\boldsymbol{\omega}^{\rm FX}_{t,n}\)} into a bitstream, denoted by \(\boldsymbol{u}_{t,n} \in \mathbb{B}^{23M \times 1}\), where $\mathbb{B}=\{0,1\}$ is the set of binary numbers.
The number of bits to be transmitted is $23M$ per local model; cf. $32M$ under \textbf{binary32}.

\subsection{Bit Flipping and Model Transmission}

Two types of noises are experienced when the fixed-point bitstream of a local model, \(\boldsymbol{u}_{t,n} \in \mathbb{B}^{23M \times 1}\), is uploaded from a device $n$ to the server: The artificial bit-flipping noise imposed by the clients, and the inherent communication noise. Both artificial and communication noises are imposed into the bitstreams, $\boldsymbol{u}_{t,n},\,\forall n$, that are generated from the fraction parts of the transformed fixed-point model parameters. The floating-point-to-fixed-point conversion eliminates the need to transmit the sign and exponent bits of the floating-point numbers, ensuring those bits are unaffected by communication noise.

 \subsubsection{Artificial Noise}\label{artificial_noise}
 Each device $n$ perturbs the bitstream \(\boldsymbol{u}_{t,n}\) by flipping the bits at random with a pre-defined bit-flipping probability, denoted by $p_{t,n,\mathrm{A}}$. The resulting flipped bitstream is denoted as \begin{align}%
  \label{art_noise}    \tilde{\boldsymbol{u}}_{t,n}={\boldsymbol{u}}_{t,n}\otimes \boldsymbol{e}_{t,n,\mathrm{A}}, \quad\boldsymbol{e}_{t,n,\mathrm{A}}\sim \mathcal{B}(23M, 1-p_{t,n,\mathrm{A}}),%
 \end{align}%
 where $\mathcal{B}(23M, 1-p)$ is the binomial distribution with length $23M$ and probability $(1-p)$; 
and $\otimes$ is the element-wise exclusive-NOR operator.
 Device $n$ encodes and modulates $\tilde{\boldsymbol{u}}_{t,n}$ into $\boldsymbol{s}_{t,n}$, which is a sequence of symbols to be transmitted to the server. 
 The length of $\boldsymbol{s}_{t,n}$ depends on the modulation-and-coding scheme employed, more specifically, the transmit rate of the scheme, \eg $rM$ for 
 $r$-$\gamma$QAM with $r$ being the coding rate and $\gamma$ being the modulation order of QAM/QPSK signals\footnote{Channel coding is typically used to combat 
 fading and noise. 
 It is optional here as it weakens the privacy-preserving capability of wireless channels.}. $\boldsymbol{s}_{t,n}\in \mathbb{R}^{\frac{23M}{r\gamma}\times 1}$. 
 
  \subsubsection{Communication Noise}\label{sec_comm_noise}

Suppose the wireless channels experience block fading; i.e., the channel coefficients remain constant during a round but may vary independently between rounds. The system bandwidth is \(B_n\) (in Hz)  in the \(t\)-th communication round.
Consider a baseline setting with a single-antenna server and single-antenna devices. At the server, the received sequence of symbols from device \(n\) in the \(t\)-th communication round can be modeled as:
\[
\boldsymbol{y}_{t,n} = {h}_{t,n} \sqrt{\beta_{t,n}} \boldsymbol{s}_{t,n} + \boldsymbol{\theta}_{t,n},
\]
where
    \({h}_{t,n}\) is the complex-valued channel coefficient for device \(n\) in the \(t\)-th communication round,
    \(\beta_{t,n}\) denotes the transmit power of device \(n\) during the \(t\)-th round, 
    and
    \(\boldsymbol{\theta}_{t,n}\in \mathbb{R}^{\frac{23M}{r\gamma}\times 1}\) is the communication noise undergone at the server (or AP), with power spectral density \(N_0\) (Watts/Hz).

On receiving \(\boldsymbol{y}_{t,n}\), the server decodes and demodulates the signal to recover the transmitted signal sequence \(\boldsymbol{s}_{t,n}\). The demodulation process converts the received signal sequence into 
a bitstream, denoted by \(\hat{\boldsymbol{u}}_{t,n}\in \mathbb{B}^{23M \times 1}\).
Bit errors can occur due to communication noises and channel impairments. \(\tilde{\boldsymbol{u}}_{t,n}\) and \(\hat{\boldsymbol{u}}_{t,n}\) likely differ. 

The BER, denoted by \(p_{t,n,\mathrm{C}}\), can be obtained based on the prior knowledge of the channel parameters, \eg $h_{t,n},\beta_{t,n},N_{0}$ and $B_n$, as well as the coding-and-modulation schemes (\eg BPSK, QPSK, \etc.) and channel conditions (\eg AWGN, Rayleigh fading, \etc.)~\cite{simon2005digital}.  
Taking BPSK and QPSK over AWGN channel as an example, the signal-to-noise ratio (SNR) is $\gamma_{t,n}=\frac{(h_{t,n})^2\beta_{t,n}}{N_{0}B_n}$, and the communication BER is $p_{t,n,\mathrm{C}}=Q(\sqrt{2\gamma_{t,n}})$~\cite{goldsmith_2005},
where $Q(x)=\frac{1}{\sqrt{2\pi}}\int_{x}^{\infty}e^{-\frac{x^{2}}{2}}dx$ is the error function~\cite{simon2005digital}.
Note that our analysis of the proposed mechanism is not limited to a specific channel model. 

As a result, the received bitstream is denoted by
\begin{align}%
\label{com_noise}     \hat{\boldsymbol{u}}_{t,n}=\tilde{\boldsymbol{u}}_{t,n}\otimes \boldsymbol{e}_{t,n,\mathrm{C}}, \quad\boldsymbol{e}_{t,n,\mathrm{C}}\sim \mathcal{B}(23M, 1-p_{t,n,\mathrm{C}}).%
 \end{align}

\subsubsection{End-to-End Bit Error}

From \(\boldsymbol{u}_{t,n}\) to \(\hat{\boldsymbol{u}}_{t,n}\), a bit may flip up to twice due to the artificial bit-flipping noise and communication noise. A bit is erroneous if flipped once. As a result, the end-to-end BER, denoted by $p_{t,n}$, is given by
\begin{align}
p_{t,n} =& \; p_{t,n,\mathrm{C}}(1 - p_{t,n,\mathrm{A}}) + p_{t,n,\mathrm{A}}(1 - p_{t,n,\mathrm{C}}) \nonumber\\
       =& \; p_{t,n,\mathrm{C}} + p_{t,n,\mathrm{A}} - 2  p_{t,n,\mathrm{C}}  p_{t,n,\mathrm{A}}.
\end{align}
The BER from the artificial bit-flipping noise, \(p_{t,n,\mathrm{A}}\), is
\begin{align}
p_{t,n,\mathrm{A}} = \frac{p_{t,n} - p_{t,n,\mathrm{C}}}{1 - 2 p_{t,n,\mathrm{C}}}.\label{Art_ber}
\end{align}
Note that the required end-to-end BER, $p_{t,n}$, corresponds to the privacy requirement specified by the privacy budget of the devices. With a given modulation-and-coding scheme and channel condition, the BER stemming from the wireless transmissions of the local models, \ie $p_{t,n,C}$, can be measured through the reciprocity of wireless channels. Given $p_{t,n}$ and $p_{t,n,C}$, $p_{t,n,A}$ can be accordingly specified.

\smallskip
We now formalize the proposed bit-flipping mechanism with respect to the end-to-end bit-flipping probability (or, in other words, end-to-end BER), $p$, as follows.
\begin{define}[\textbf{Bit-flipping DP Mechanism}]\label{def-BFDP}
For a local dataset $\mathcal{D}$ and an encoded bit stream of a model parameter, \ie $\boldsymbol{u}(\cdot) \in \mathbb{B}^{23M\times 1}$, the randomized bit-flipping DP mechanism is defined as
\begin{equation}
    \mathcal{M}_{\rm BF}(\mathcal{D}, \boldsymbol{u}(\cdot), {\bar\kappa}_{\Delta\boldsymbol{\omega}_{\max,n}}, \lambda, \epsilon) = \boldsymbol{u}(\mathcal{D}) \otimes \mathcal{B}(23M, 1-p),
\end{equation}
where $\mathcal{B}(23M, 1-p)$ is the binomial distribution with length $23M$ and probability $(1-p)$; 
${\bar\kappa}_{\Delta\boldsymbol{\omega}_{\max,n}}$ is the expected bit-level distance under a given classical sensitivity $\Delta\boldsymbol{\omega}_{\max,n}$; $\otimes$ is the element-wise exclusive-NOR (XNOR) operator.
\end{define}

\subsection{Model Recovery}\label{sec:recovery}

With the prior knowledge about the sign and exponent parts of the fixed-pointed numbers, the server can first recover the local models after demodulating and decoding the bitstream \(\hat{\boldsymbol{u}}_{t,n}\). This is done by concatenating the sign and exponent parts of the fixed-point number to  \(\hat{\boldsymbol{u}}_{t,n}\), followed by subtracting $3\times2^{(c_{30}c_{29}\cdots c_{23})_{2}-126}$ from the decimal fixed-point elements to recover the floating-point local models. 
The recovered local model parameters of device $n$ are denoted by $\tilde{\boldsymbol{\omega}}_{t,n}$.

\subsection{Implementation Details}\label{algorithm_detail}
\textbf{Algorithm 1} summarizes the proposed channel-native bit-flipping DP mechanism.
Specifically, after finishing the local training in Line 6, the encoding process starts. As described in Lines 7 to~12, at the beginning of the encoding process, the clients convert the \textbf{binary32} floating-point model parameters $\boldsymbol{\omega}_{t,n}$ to fixed-point numbers by adding a constant generated from $\nu_{\infty}=(c_{31}c_{30}...c_{0})_{\rm 2}$; see \eqref{fp2fx}. Then, the clients remove the common sign part (i.e., $1$) and exponent part (i.e., $(c_{30}c_{29}\cdots c_{23})_2+2$) of the fixed-point numbers, and encode the fraction part into a bitstream $\boldsymbol{u}_{t,n}$ (see Section~\ref{subsection_map}).
The clients also estimate the BERs \(p_{t,n,\mathrm{A}}\); see Line 13.
Based on the BERs and privacy requirements, the clients calculate the artificial bit-flipping probabilities needed to supplement the inherent privacy provided by the channels; see \eqref{Art_ber}.

    Then, artificial and communication bit-flipping operations, performed by the clients and caused by the noisy channels, are sequentially added to the bitstream to enhance privacy. The former is done according to \eqref{art_noise} before transmission; see Line~14. The latter comes from the noisy communication channel according to \eqref{com_noise}; see Line 15.
       After receiving all bitstreams \(\{\hat{\boldsymbol{u}}_{t,n},\forall t,n\}\) from the clients, the server reconstructs the fixed-point numbers with the common sign and exponent parts (i.e.,~1 and $(c_{30}c_{29}\cdots c_{23})_2+2$), and subtracts the common value to recover the original \textbf{binary32}  model parameters \(\tilde{\boldsymbol{\omega}}_{t,n}\)  (see Sec. \ref{sec:recovery}), as shown in Line~18. 
    The received models are aggregated at the server using \eqref{eq:local update 2} in Line~19.

\section{Privacy Analysis}
\label{sec:privacy}

Our approach adapts standard privacy analysis to operate at the bit level. Specifically, the sensitivity of the model parameters, $\Delta\boldsymbol{\omega}_{\max}$, is first evaluated, as in conventional DP algorithms.
Then, the bit-level distances between parameter pairs achieving $\Delta\boldsymbol{\omega}_{\max}$ are analyzed statistically.
Finally, the Renyi divergence is analyzed based on these binary distances instead of the original parameter values.
This adaptation ensures privacy guarantees are upheld while capturing the nuances of parameter representation at the bit level.

\begin{algorithm}
\caption{Channel-native bit-flipping DP mechanism}\label{Algorithm 2}
\textbf{Initialization: }Clipping threshold~$G$; maxima of $\ell_{\infty}$-norm $\nu_{\infty}$; initial global model~$\boldsymbol{\omega}_{0,\mathcal{G}}$; number of local epochs~$E$; total number of training iterations $T$. 
\begin{algorithmic} [1]
\FOR{$t=1,\dots, T$}
\FOR{all devices $n\in\mathcal{N}$ in parallel} 
\IF{$t\notin \mathcal{I}_E$}
\STATE {Perform local training with (\ref{eq:local update 1}) and obtain~$\boldsymbol{\omega}_{t,n}$;}
\ELSE
\STATE Perform local training with (\ref{eq:local update 1}) and obtain~$\boldsymbol{\omega}_{t,n}$;
\STATE Encode a local model $\boldsymbol{\omega}_{t,n}$ into a bitstream~$\boldsymbol{u}_{t,n}$, as follows:
\STATE \textit{\small Initialize: $\boldsymbol{\omega}^{\mathrm{FP}}_{t,n}=\boldsymbol{\omega}_{t,n}$, $\nu_{\infty}$;}
\STATE \textit{\small Obtain the binary format of $\nu_{\infty}=(c_{31}c_{30}...c_{0})_{\rm 2}$;}
\STATE \textit{\small Perform floating-point-to-fixed-point conversion with \eqref{fp2fx} and obtain $\boldsymbol{\omega}^{\mathrm{FX}}_{t,n}$;}
\STATE \textit{\small Remove the common sign and exponent parts of $\boldsymbol{\omega}^{\mathrm{FX}}_{t,n}$, \ie $1$ and $(c_{30}c_{29}\cdots c_{23})_2+2$;}
\STATE \textit{\small Encode the fraction part of $\boldsymbol{\omega}^{\mathrm{FX}}_{t,n}$ and obtain $\boldsymbol{u}_{t,n}$;}  
\STATE Calculate the artificial BER $p_{t,n,\mathrm{A}}$;
\STATE Add artificial bit-flipping noise to the bitstream $\boldsymbol{u}_{t,n}$ to obtain $\tilde{\boldsymbol{u}}_{t,n}$, see \eqref{art_noise};
\STATE Modulate and upload $\tilde{\boldsymbol{u}}_{t,n}$ to the server through the noisy wireless channel, see \eqref{com_noise};
\ENDIF
\ENDFOR
\STATE  The server receives signals from all devices, obtains $\hat{\boldsymbol{u}}_{t,n}$, and recovers $\tilde{\boldsymbol{\omega}}_{t,n}$.
\STATE The server aggregates all local models and obtains~$\boldsymbol{\omega}_{t,\mathcal{G}}$, see \eqref{eq:local update 2}. 
\ENDFOR
\end{algorithmic}  
\end{algorithm}

\subsection{Bit-Level Distance}\label{Section:sensitivity}

When \(t\in\mathcal{I}_E\), a certain level of superimposed artificial noise and communication noise, resulting in the end-to-end BER \(p_{t,n}\), is imposed to the model parameter \(\boldsymbol{\omega}_{t,n}(\mathcal{D}_n)\) of every device \(n\). This results in a perturbed model parameter \(\tilde{\boldsymbol{\omega}}_{t,n}(\mathcal{D}_n) = \boldsymbol{\omega}_{t,n}(\mathcal{D}_n) + \boldsymbol{z}_{t,n}\) at the server.

Suppose that \(\boldsymbol{\omega}_{t,n}(\mathcal{D}_n)\) and \(\boldsymbol{\omega}_{t,n}(\mathcal{D}'_n)\) are derived from two adjacent datasets \(\mathcal{D}_n\) and \(\mathcal{D}'_n = \mathcal{D}_n \cup \{d'\}\), 
respectively. $d'$ is a data sample. 
To measure the sensitivity between these two sets of model parameters, we use the \(\ell_2\)-norm. The (classical) sensitivity concerning function $\boldsymbol{\omega}_{t,n}(\mathcal{D}_n) $ is given by
\begin{align}\label{eq: classical sensitivity}
\!\!\!\!\Delta \boldsymbol{\omega}_{\max,n}\!\! =\!\! \underset{\forall t, \mathcal{D}_n, \mathcal{D}'_n}{\max} \Vert \boldsymbol{\omega}_{t,n}(\mathcal{D}_n) - \boldsymbol{\omega}_{t,n}(\mathcal{D}'_n) \Vert_2
=\frac{2{\eta G}}{|\mathcal{D}_n|},
\end{align}
which is obtained from
\setlength{\abovedisplayskip}{4pt}
\setlength{\belowdisplayskip}{4pt}
\begin{subequations}
\begin{align}
    \notag &\big\Vert \boldsymbol{\omega}_{t,n}(\mathcal{D}_n) - \boldsymbol{\omega}_{t,n}(\mathcal{D}'_n) \big\Vert_2\\
    &= \Big\Vert \left( \tfrac{\eta}{|{\cal D}_n|} - \tfrac{\eta}{|{\cal D}_n'|} \right)
    {\sum}_{d\in{\cal D}_{n}} \nabla f(d,\mathbf{w})\!\cdot\!\min\!\left\{1,\tfrac{G}{\|\nabla f(d,\mathbf{w})\|_2}\right\} \notag\\
    &\quad - \tfrac{\eta}{|{\cal D}_n'|} \nabla f(d',\mathbf{w})\!\cdot\!\min\!\left\{1,\tfrac{G}{\|\nabla f(d',\mathbf{w})\|_2}\right\} \Big\Vert_2 \label{sensitivity_b}\\
    &\leq \left\Vert \tfrac{\eta}{|{\cal D}_n|(|{\cal D}_n|+1)}
    {\sum}_{d\in{\cal D}_{n}} \nabla f(d,\mathbf{w})\!\cdot\!\min\!\left\{1,\tfrac{G}{\|\nabla f(d,\mathbf{w})\|_2}\right\} \right\Vert_2 \notag\\
    &\quad + \left\Vert \tfrac{\eta}{|{\cal D}_n|+1} \nabla f(d',\mathbf{w})\!\cdot\!\min\!\left\{1,\tfrac{G}{\|\nabla f(d',\mathbf{w})\|_2}\right\} \right\Vert_2 \label{sensitivity_c}\\
    &\leq \tfrac{\eta}{|{\cal D}_n|(|{\cal D}_n|+1)} {\sum}_{d\in{\cal D}_{n}} 
    \left\Vert \nabla f(d,\mathbf{w})\!\cdot\!\min\!\left\{1,\tfrac{G}{\|\nabla f(d,\mathbf{w})\|_2}\right\} \right\Vert_2 \notag\\
    &\quad + \tfrac{\eta}{|{\cal D}_n|+1} 
    \left\Vert \nabla f(d',\mathbf{w})\!\cdot\!\min\!\left\{1,\tfrac{G}{\|\nabla f(d',\mathbf{w})\|_2}\right\} \right\Vert_2 \label{sensitivity_d}\\
    &\leq \tfrac{2\eta G}{|\mathcal{D}_n|}, \label{sensitivity_e}
\end{align}%
\end{subequations}%
where 
\eqref{sensitivity_b} is based on \eqref{eq:local update 1} and reorganization, \eqref{sensitivity_c} is due to the homogeneity of $\ell_2$-norm and that $|{{\cal D}}_n'|=|{{\cal D}}_n|+1$, \eqref{sensitivity_d} is due to the homogeneity of $\ell_2$-norm, and \eqref{sensitivity_e} is based on $\frac{1}{|{{\cal D}}_n|+1}\leq\frac{1}{|{{\cal D}}_n|}$ and $\big \Vert\nabla{f}(d,\mathbf{w})\cdot \min \{1,\frac{G}{\Vert\nabla{f}(d,\mathbf{w})\Vert_2}\}\big\Vert_2\leq G$.

 \subsubsection{Definition of Expected Bit-Level Distance}\label{subsubsection_def_sensitivity}
Since we consider perturbation upon the bitstreams \(\boldsymbol{u}_{t,n}\) (through artificial bit-flipping and communication noises), which is also a function of \(\mathcal{D}_n\), \ie \(\boldsymbol{u}_{t,n}=\boldsymbol{u}_{t,n}(\mathcal{D}_n)\), we are interested in the difference between the pair of bitstreams corresponding to~$\Delta \boldsymbol{\omega}_{\max,n} $ {before the perturbation}. It is important to note that the difference is non-deterministic and influenced by various sources of inherent randomness, including dropout mechanisms in the CNN model, random seeds, and the stochastic padding algorithm defined in the IEEE~754 standard. We amount to the expectation of the bit-level difference corresponding to $\Delta \boldsymbol{\omega}_{\max,n} $ and accordingly define a new metric, which will be employed in the privacy analysis of the proposed mechanism in Section~\ref{renyi_DP_privacy_loss}.

Consider two bitstreams \(\boldsymbol{u}_{t,n}(\mathcal{D}_n)\) and \(\boldsymbol{u}_{t,n}(\mathcal{D}'_n)\), each derived from the respective model parameters,
$\boldsymbol{\omega}_{t,n}(\mathcal{D}_n)$ and $\boldsymbol{\omega}_{t,n}(\mathcal{D}'_n)$, with their difference equal to the classical sensitivity, $\Delta \boldsymbol{\omega}_{\max,n}\equiv \Vert \boldsymbol{\omega}_{t,n}(\mathcal{D}_n)-\boldsymbol{\omega}_{t,n}(\mathcal{D}'_n)\Vert_2$.
 The bit-level distance between $\boldsymbol{u}_{t,n}\triangleq\boldsymbol{u}_{t,n}(\mathcal{D})$ and $\boldsymbol{u}_{t,n}'\triangleq\boldsymbol{u}_{t,n}(\mathcal{D}')$ is defined by \begin{align}
\notag\kappa(\boldsymbol{\omega}_{t,n}(\mathcal{D}_n),\boldsymbol{\omega}_{t,n}(\mathcal{D}'_n))\triangleq \Vert \boldsymbol{u}_{t,n}(\mathcal{D}_n) - \boldsymbol{u}_{t,n}(\mathcal{D}'_n) \Vert_1  
\\={\sum}_{k=1}^{23M}2^{(k{\,\rm \,mod} 23)-23}|u_k-u_k'|.\notag
\end{align}

Due to inherent randomness, the bit-level distance between $\boldsymbol{u}_{t,n}$ and $\boldsymbol{u}_{t,n}'$ is non-deterministic, even with fixed $\mathcal{D}$, $\mathcal{D}'$, and $\boldsymbol{\omega}_{t-1,n}$. It may occasionally reach the maximum value—an extreme and random case where all bits differ—yet the same configurations can still yield varying distances across repeated experiments.

Consider a single bit ${u}_k$, such as the $k$-th bit in the bitstream $\boldsymbol{u}_{t,n} \in \mathbb{B}^{23M \times 1}$, which is also the $j$-th fraction bit of a real number, where $j=(k {\,\rm \,mod}\, 23)$. Due to inherent randomness, the bit-level difference between ${u}_k$ and ${u}_k'$ attains its maximum value (i.e., $2^{(k{\,\rm \,mod}\, 23)-23}$) with probability $q_k$, and equals zero with probability $1 - q_k$. The expectation of its bit-level distance between ${u}_k$ and ${u}_k'$ is $2^{(k{\,\rm \,mod}\, 23)-23}q_k$.

We now introduce a new metric, referred to as the sum of the expected bit-level distance of all bits between $\boldsymbol{u}_{t,n}$ and $\boldsymbol{u}'_{t,n}$ under a given classical sensitivity $\Delta \boldsymbol{\omega}_{\max,n}$, denoted as 
\begin{equation}
{\bar\kappa}_{\Delta\boldsymbol{\omega}_{\max,n} } =\mathbb{E}(\kappa)={\sum}_{k=1}^{23M}2^{(k{\,\rm \,mod} 23)-23}q_k.\label{expected_bit_level_distance}%
\end{equation}%
\begin{remark}~\label{remark expectation}
     While the bit-level distances of these parameter pairs can vary significantly, ${\bar\kappa}_{\Delta\boldsymbol{\omega}_{\max,n} }$ is introduced to characterize this variability. In the next section, we leverage Rényi divergence to quantify the privacy leakage risk, which has been proven to positively correlate with ${\bar\kappa}_{\Delta\boldsymbol{\omega}_{\max,n}}$.
     
Note that directly using the maximum bit-level distance may overestimate the privacy risk, as it fails to capture the full distributional characteristics. This coarse approximation could lead to excessive differential privacy noise injection, thereby impairing the utility of the federated learning model.
\end{remark}
\begin{remark}\label{remark kappa}
For two bitstreams, the bit-level distance $\kappa$ differs from the classical $\ell_1$-norm of the distance between the bitstreams. 
Consider two floating-point numbers $\boldsymbol{a}$ and $\boldsymbol{b}$ with the same sign and exponent parts, and different fraction parts, denoted by $(a_{22}\cdots a_{0})_2$ and $(b_{22}\cdots b_{0})_2$, respectively. 

The $\ell_1$-norm is $|\sum_{j=0}^{22}2^{j-23}(a_{j}-b_{j})|$. By contrast, the bit-level distance is $\kappa=\sum_{j=0}^{22}2^{j-23}|a_{j}-b_{j}|$, and is no smaller than the $\ell_1$-norm.
As a consequence, given $\Delta\boldsymbol{\omega}_{\max,n}$, ${\bar\kappa}_{\Delta\boldsymbol{\omega}_{\max,n} }$ is no smaller than the expectation of the $\ell_1$-norm of the distance between any two bitstreams.
\end{remark}

\subsubsection{Evaluation of ${\bar\kappa}_{\Delta\boldsymbol{\omega}_{\max,n} }$}
To evaluate ${\bar\kappa}_{\Delta\boldsymbol{\omega}_{\max,n} }$, we define $\boldsymbol{x} \in \mathbb{R}^{M\times1}$ with  
$\boldsymbol{x}= \boldsymbol{\omega}_{t,n}(\mathcal{D}_n)-\boldsymbol{\omega}_{t,n}(\mathcal{D}'_n)$. 
 Let $x_m$ denote the $m$-th element of $\boldsymbol{x}$.
Hence, $\Vert\boldsymbol{x}\Vert_2=\sqrt{\sum_{m=1}^{M}|x_m|^2}=\Delta\boldsymbol{\omega}_{\max,n}$. 
\begin{figure}
    \centering
    \includegraphics[scale=0.7]{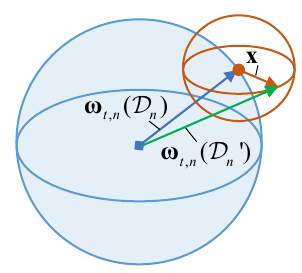}
    \caption{An illustration of the model parameters $\boldsymbol{\omega}_{t,n}(\mathcal{D}_n)$, $\boldsymbol{\omega}_{t,n}(\mathcal{D}_n')$, and $\boldsymbol{x}$. The blue sphere denotes the domain of the model parameters $\boldsymbol{\omega}_{t,n}(\mathcal{D}_n)$. The red sphere denotes the domain of $\boldsymbol{x}$.}
    \label{fig:sensitivity_domain}
\end{figure}
By definition, ${\bar\kappa}_{\Delta\boldsymbol{\omega}_{\max,n} }$
can be written as
    \begin{align}
&\!\!\!\!{\bar\kappa}_{\Delta\boldsymbol{\omega}_{\max,n} }\!\!=\!\!\frac{\int\!\!\cdots\!\!\int_{ \mathcal{R}}\!\!\kappa(\boldsymbol{\omega}\!\!+\!\!\boldsymbol{x},\boldsymbol{\omega})\,d\omega_{1}\cdots d\omega_{M}d x_{1}\cdots  d x_{M}}{\int\cdots\int_{ \mathcal{R}} 1\,d\omega_{1}\!\cdots\! d\omega_{M} d x_{1}\!\cdots \!d x_{M}},\label{binary_sensitivity}
\end{align}
where $\mathcal{R} = \{(\omega_{1},\cdots,\omega_{M},x_1,\cdots,x_M)\}$ stands for the $2M$-dimensional vector space satisfying
\begin{itemize}
    \item The $m$-th element, $\omega_m$, of the model parameter $\boldsymbol{\omega}$ is in the range $\omega_m\in[-2^{(c_{30}c_{29}...c_{23})_{2}-126},2^{(c_{30}c_{29}...c_{23})_{2}-126}],\,\forall m$;
    \item  $\Vert\boldsymbol{x}\Vert_2=\Delta\boldsymbol{\omega}_{\max,n}$ achieves the classical sensitivity.
\end{itemize}
The nominator on the right-hand side (RHS) of~\eqref{binary_sensitivity} gives the volume of $\mathcal{R}$, and its reciprocal specifies the density, i.e., the probability distribution function (PDF), of~$\mathcal{R}$.

Fig.~\ref{fig:sensitivity_domain} illustrates the defined vector space $\mathcal{R}$, using spheres to conceptualize regions defined by the \(\ell_1\)-norm for easier visualization. The volume of $\mathcal{R}$ can be visualized as the total area of all the red spheres, each centered at every possible pair of vectors $\boldsymbol{\omega}$ and~$\boldsymbol{x}$.
According to~\eqref{eq: classical sensitivity}, $\sqrt{\sum_{m=1}^{M}|x_m|^2}= \frac{2{\eta G}}{|\mathcal{D}_n|}$. ${\bar\kappa}_{\Delta\boldsymbol{\omega}_{\max,n} }$ can be evaluated empirically at each device $n$ based on its dataset $\mathcal{D}_n$; see Section~\ref{section:binary_sensitivity}.

\subsection{End-to-End \((\lambda,\epsilon)\)-Rényi DP}\label{renyi_DP_privacy_loss}

{  At the server, $\hat{\boldsymbol{u}}_{t,n}(\mathcal{D}_n)$ and $\hat{\boldsymbol{u}}_{t,n}(\mathcal{D}'_n)$ are the received bitstreams perturbed under the proposed bit-flipping DP mechanism with end-to-end BER $p_{t,n}$.  The following theorem evaluates the end-to-end BER, denoted by $p_{t,n}$, ensuring that applying the bit-flipping DP mechanism to the binary model parameters of client $n$ at the $t$-th round satisfies $(\lambda, \frac{\epsilon}{K})$-Rényi DP. Furthermore, the proposed mechanism satisfies $(\lambda, {\epsilon})$-Rényi DP over $K$ communication rounds.

\begin{theo}\label{theo1_renyi}
With the expected bit-level distance ${\bar\kappa}_{\Delta\boldsymbol{\omega}_{\max,n}}$ defined, applying the proposed bit-flipping mechanism $\mathcal{M}_{\rm BF}(\boldsymbol{u}(\mathcal{D}_n),p_{t,n})$ defined in \textbf{Definition~\ref{def-BFDP}} for $K$ communication rounds satisfies $(\lambda,\epsilon)$-Rényi DP when the end-to-end BER is $p_{t,n}\geq\left(1 +
\Big( \tfrac{(\lambda-1)\epsilon}{K {\bar\kappa}_{\Delta\boldsymbol{\omega}_{\max,n}}} \right)^{\frac{1}{\lambda-1}}
\Big)^{-1}$.
\end{theo}

\begin{proof}
   The detailed proof can be found in \textbf{Appendix~\ref{proof_theo1}}.
\end{proof}

In Theorem \ref{theo1_renyi}, we analyze the privacy leakage risk using the Rényi divergence, based on the distributions of the received bitstreams corresponding to two adjacent datasets. To quantify the influence of bit-level deviations, we introduce the expectation-based metric ${\bar\kappa}_{\Delta\boldsymbol{\omega}_{\max,n}}$, which captures the joint impact of the maximum bit-level distance and its probability on privacy leakage. We further demonstrate that the Rényi divergence is positively correlated to ${\bar\kappa}_{\Delta\boldsymbol{\omega}_{\max,n}}$, by establishing an upper bound involving the one-round privacy budget $\frac{\epsilon}{K}$. This facilitates analyzing the end-to-end BER~$p_{t,n}$.}

\begin{remark}
The above analysis does not consider encryption. Nevertheless, it can be extended to systems where devices encrypt their local models before bit-flipping and transmissions.

In the case of non-diffusion encryption methods, \eg certain stream ciphers and the one-time pad, each bit or byte of the plaintext is directly transformed through linear operations, e.g., XOR. Bit errors in the ciphertext mirror those in the corresponding positions of the decrypted plaintext. The BER of the ciphertext is consistent with the BER of the decrypted plaintext, i.e., $p_{t,n}\triangleq p^{\rm PT}_{t,n} = p^{\rm CT}_{t,n}$ with $p^{\rm PT}_{t,n}$ and $p^{\rm CT}_{t,n}$ being the BERs of the plaintext and ciphertext, respectively.

In the case of diffusion encryption methods, \eg AES and DES, a fixed block size of 128 or 64 bits is typically used. The methods are engineered to propagate the impact of bit errors across entire blocks. Even a single bit error can corrupt an entire block, leading to widespread bit errors. The BER of the plaintext within the block is expected to be $0.5$ to make the original plaintext irrecoverable. 
Taking AES with a block size of 128 bits as an example. When the BER of the ciphertext is $p^{\mathrm{CT}}_{t,n}$, the block error rate of a received block is $1-(1-p^{\mathrm{CT}}_{t,n})^{128}$. Then, the BER of the decrypted plaintext is $0.5$ when one or more bit errors are in the received block of the ciphertext. When there is no error in the ciphertext block, the BER of the decrypted plaintext is $0$. Based on the Law of Total Probability, the BER of the decrypted plaintext $p_{t,n}$ is 
\begin{align}
p_{t,n}\!\triangleq \!p^{\rm PT}_{t,n}\!=\!(1\!-\!(1\!-\!p^{\mathrm{CT}}_{t,n})^{128})\!\times \!0.5\!+\!(1\!-\!p^{\mathrm{CT}}_{t,n})^{128}\!\times \!0.
\end{align}
According to~\eqref{eq:Joint_BER_preserving_privacy}, we have $p^{\mathrm{CT}}_{t,n}=1-(1-2p_{t,n})^{\frac{1}{128}}$.

In both cases, the BERs of the decrypted plaintext, $p^{\rm PT}_{t,n}$, can be evaluated and substituted into~\eqref{eq:Joint_BER_preserving_privacy} to help specify the end-to-end BER of the proposed bit-flipping DP mechanism.
\end{remark}

\section{Convergence Analysis} 
\label{sec:convergence}

In this section, we analyze the convergence of WFL under the proposed bit-flipping mechanism. 
Due to the artificial bit-flipping noise and communication noise, $\tilde{\boldsymbol{\omega}}_{t,n},\,\forall n$ is likely to be imprecise. We define $\boldsymbol{z}_{t,n}$ as the local bias between the locally updated model of device $n$ and the corresponding model recovered at the server, \ie $\boldsymbol{\omega}_{t,n}$ and $\tilde{\boldsymbol{\omega}}_{t,n}$: 
\begin{align}
   \boldsymbol{z}_{t,n}= \tilde{\boldsymbol{\omega}}_{t,n}-\boldsymbol{\omega}_{t,n}.\label{eq:z_tn}
\end{align}
The global bias between the ideal global model, \ie $
\sum_{n\in\mathcal{N}}q_n\boldsymbol{\omega}_{t,n}$, and the corresponding imperfect global model, \ie $\boldsymbol{\omega}_{t,\mathcal{G}}\!=\!
\sum_{n\in\mathcal{N}}q_n\tilde{\boldsymbol{\omega}}_{t,n}$, can be defined as
\begin{align}
    \boldsymbol{z}_{t,G}={\sum}_{n\in\mathcal{N}}q_n(\tilde{\boldsymbol{\omega}}_{t,n}-\boldsymbol{\omega}_{t,n})={\sum}_{n\in\mathcal{N}}q_n\boldsymbol{z}_{t,n}.\label{eq:z_tG}
\end{align}

\begin{lemma}\label{lemma:z}With the mean and variance of a bitstream with randomly flipped bits 
given in \textbf{Lemma~\ref{lemma:expectation and Var}}, the local and global biases, $\boldsymbol{z}_{t,n},\,\forall n$ and $\boldsymbol{z}_{t,\mathcal{G}}$, have the following properties: 
\begin{enumerate}
    \item The expectation and variance of $\boldsymbol{z}_{t,n}$ are given by  
\begin{align} \!\!\!\!\!\!\!\!\!\!\!\!\!\!\!\!\!\!\!\!\!\!\!\!\mathbb{E}(\boldsymbol{z}_{t,n})\approx &2p_{t,n}{\boldsymbol{\omega}}_{t,n};\label{eq:E_z_tn}\\
 \!\!\!\!\!\!\!\!\!\!\!\!\!\!\!\!\!\!\!\!\!\!\!\! \mathbb{E}(\Vert \boldsymbol{z}_{t,n}\!\!-\!\!\mathbb{E}(\boldsymbol{z}_{t,n})\Vert_2^2) 
\!=&\!\frac{(1\!\!-\!\!4^{-23})}{3}Mp_{t,n}(1\!\!-\!\!p_{t,n})\times \nonumber \\
&2^{2(c_{30}c_{29}\cdots c_{23})_{2}-250}.\label{eq:Var_z_tn}
\end{align}

\item The $\ell_2$-norm of the global bias $\boldsymbol{z}_{t,\mathcal{G}}$, \ie $\Vert\boldsymbol{z}_{t,\mathcal{G}}\Vert^2$, is upper bounded by
\begin{align}
\notag
\!\!\!\!\!\!\!\!\!\!\!\!\mathbb{E}( \Vert\boldsymbol{z}_{t,\mathcal{G}}\Vert_2^2)\leq &\frac{(1\!\!-\!\!4^{\!-\!23})M}{3}{\sum}_{n\in\mathcal{N}}q_n^2 p_{t,n}(1\!\!-\!\!p_{t,n})\times \nonumber\\
&2^{2(c_{30}c_{29}...c_{23})_{2}-250} \!\!+\!\!{\sum}_{n\in\mathcal{N}} \!\!q_n p_{t,n}^2 \nu_{2}^2\nonumber\\
\triangleq & \!X^{\rm BF}_t(G,\nu_{2},\nu_{\infty}). \label{eq:Ez^2}
\end{align}
where $X^{\rm BF}_t(G,\nu_{2},\nu_{\infty})$ depends on 
$\nu_{\infty}=(c_{31}c_{30}...c_{0})_{\rm 2}$
since \(c_{30}c_{29}\cdots c_{23}\) are the exponent bits of \(\nu_{\infty}\). $X^{\rm BF}_t(G,\nu_{2},\nu_{\infty})$ is positively related to $G$, $\nu_{2}$ and $\nu_{\infty}$.
\end{enumerate}
\end{lemma}
\begin{proof}
See \textbf{Appendix~\ref{Appendix:2}}.
\end{proof}

\begin{remark}[Comparison with the Gaussian Mechanism]
    When utilizing the Gaussian Mechanism to achieve the $(\epsilon,\delta)$-DP, the expectation and variance of $\boldsymbol{z}_{t,n}$ are 
\begin{align} 
&\!\!\!\!\!\!\!\!\!\!\!\!\mathbb{E}(\boldsymbol{z}_{t,n})=0\; \text{ and }\;
\mathbb{E}(\Vert \boldsymbol{z}_{t,n}\!\!-\!\!\mathbb{E}(\boldsymbol{z}_{t,n})\Vert_2^2) 
\!= \! M \sigma_n^2,\label{eq:Var_z_tn_gaussian}
\end{align}
where $\sigma_n=\frac{\Delta\boldsymbol{\omega_n}K\sqrt{2\ln(1.25/\delta)}}{\epsilon}$ is the Gaussian variance~\cite{dwork2014algorithmic}.

The $\ell_2$-norm of the global bias is upper bounded by
\begin{align}
& \!\!\!\!\mathbb{E}( \Vert\boldsymbol{z}_{t,\mathcal{G}}\Vert_2^2)\!\!\leq \!\!{\sum}_{n\in\mathcal{N}} q_n^2 M\frac{(\Delta\boldsymbol{\omega}_{\max,n} K)^2 2\ln(\frac{1.25}{\delta})}{\epsilon^2}\notag
 \\&  \!\!=  \!\!{\sum}_{n\in\mathcal{N}} q_n^2 M\frac{(2{\eta G} K)^2\times 2\ln(1.25/\delta)}{|\mathcal{D}_n|^2\epsilon^2} 
 \!\!\triangleq \!\!X^{\rm Gauss}_t(G).\label{eq:Ez^2_Gaussian}
\end{align}
When the privacy budget, \ie $\epsilon$, and the FL hyperparameters, \eg ${{\bar\kappa}_{\Delta\boldsymbol{\omega}_{\max} }}$, $\Delta \boldsymbol{\omega}$, and $K$,  are given, one can choose the suitable mechanism, \eg the proposed bit-flipping mechanism or the Gaussian mechanism, by numerically comparing the $\ell_2$-norm of the global bias under the two mechanisms, \ie~\eqref{eq:Ez^2} and~\eqref{eq:Ez^2_Gaussian}. The bit-flipping mechanism may enjoy a relatively larger privacy budget since it can leverage the communication noise to spare some privacy budget for bit-flipping.
\end{remark}

Next, we proceed to analyze the convergence of WFL under the proposed bit-flipping DP mechanism.
Recall a total of \(T<\infty\) training iterations, with \(E\) local iterations per communication round.
{ For illustration convenience, we consider full-batch gradient descent in this paper\footnote{ The convergence and privacy analysis under full-batch gradient descent can be extended to the mini-batch setting by accounting for the reduced exposure of individual samples and the stochastic nature of sampling~\cite{abs-1908-10530,DLDP,8455532,10542235,stich2018local}.}.}

\begin{assumption}\label{assumption}
Some widely used assumptions for analyzing convergence bounds for FL are considered:
\begin{enumerate}
		\item There exist such two constants $\mu>0$ and $\alpha>0$ that per any device $n$, the local loss function $F_n(\mathbf{w}),\forall n$ is $\mu$-strongly convex and $\alpha$-smooth, \ie $\alpha\left\Vert \mathbf{w}_{1}\!-\!\mathbf{w}_{2}\right\Vert _{2}\geq\! \left\Vert \nabla F_{n}\left(\mathbf{w}_{1}\right)\!-\!\nabla F_{n}\left(\mathbf{w}_{2}\right)\right\Vert_{2}\! \geq\! \mu\left\Vert \mathbf{w}_{1}\!-\!\mathbf{w}_{2}\right\Vert_{2} ,\forall \mathbf{w}_{1},\mathbf{w}_{2}\in \mathbb{R}^{M\times 1}$. Subsequently, the global loss function $F(\mathbf{w})$ is $\mu$-strongly convex and $\alpha$-smooth.
\item  As assumed in~\cite{Li2020On}, the variances of stochastic gradients are bounded by $\mathbb{E}_{\xi_{t,n}}\big(\Vert \nabla F_{n}(\mathbf{w}_{t,n},\xi_{t,n})- \nabla F_{n}(\mathbf{w}_{t,n})\Vert_2^2\big)\leq \sigma^2_n, \forall n,t$, where $\xi_{t,n}$ is the mini-batch of device $n$ in the $i$-th iteration of round $t$. Considering full-batch gradient, in this paper, $\sigma^2_n=0$.
\end{enumerate}
\end{assumption}
\begin{define}
\label{def-divergence}
For convergence analysis,  
\begin{enumerate}
    \item Define two virtually aggregated models at the $t$-th iteration as $\bar{\boldsymbol{\omega}}_t={\sum}_{n\in\mathcal{N}}q_n\boldsymbol{\omega}_{t,n}$ and $\mathbf{\bar{w}}_t={\sum}_{n\in\mathcal{N}}q_n\mathbf{w}_{t,n}$; see~\eqref{eq:local update 1} and~\eqref{eq:local update 2}, respectively.
    $\bar{\boldsymbol{\omega}}_t=\bar{\mathbf{w}}_t$, if $t\notin\mathcal{I}_E$;  $\bar{\mathbf{w}}_{t}=\bar{\boldsymbol{\omega}}_{t}+\boldsymbol{z}_{t,\mathcal{G}}$, if $t\in\mathcal{I}_E$.
    \item Let $\Gamma=F^*-{\sum}_{n\in\mathcal{N}}q_n F_n^*$ be the global heterogeneity degree to capture the heterogeneity of the local datasets.
\end{enumerate}
\end{define}

Under \textbf{Assumption 1}, the convergence upper bound of WFL under the proposed bit-flipping mechanism is established by analyzing $\Vert\bar{\mathbf{w}}_t-\mathbf{w}^*\Vert_2^2$, which is the distance between $\bar{\mathbf{w}}_t$, \ie the virtually aggregated model at the $t$-th training iteration, and $\mathbf{w}^*$, \ie the optimal global model, as stated in the following.

\begin{theo}\label{theo:convergence}
Given $T$ and $E$, the convergence upper bound of WFL under the bit-flipping mechanism is given by 
    \begin{align}
      \!\!\!  \!\!\!\mathbb{E}\!\left(\!\Vert\bar{\mathbf{w}}_{T}\!\!-\!\!\mathbf{w}^*\Vert_{2}^{2}\right)\!\leq\!\zeta_{1}\Vert\bar{\mathbf{w}}_{0}\!\!-\!\!\mathbf{w}^*\Vert_{2}^{2}\!\!+\zeta_{2}\eta^{2}\!B\!+\!\zeta_{3}\frac{ X^{\rm BF}(G,\nu_2,\nu_{\infty})}{\sqrt{p_{\max}}},\label{eq:bound_T}
    \end{align}
where the expectation $\mathbb{E}$ is taken over the random bit-flipping noise across all communication rounds; 
\begin{align}
  \notag  \zeta_1=&(1+\sqrt{p_{\max}})^{K}(1-\eta\mu)^{T};
    \\\notag 
    \zeta_2=&\frac{1\!\!-\!\!(1\!\!-\!\!\eta\mu)^{T}(1\!\!+\!\!\sqrt{p_{\max}})^{K}}{1\!\!-\!\!(1\!\!-\!\!\eta\mu)^{E}(1\!\!+\!\!\sqrt{p_{\max}})}(1\!\!+\!\!\sqrt{p_{\max}})\frac{1\!\!-\!\!(1\!\!-\!\!\eta\mu)^{E}}{\eta\mu};\\\notag 
    \zeta_3=&\frac{1\!\!-\!\!(1\!\!-\!\!\eta\mu)^{T}(1\!\!+\!\!\sqrt{p_{\max}})^{K}}{1\!\!-\!\!(1\!\!-\!\!\eta\mu)^{E}(1\!\!+\!\!\sqrt{p_{\max}})}(1\!\!+\!\!\sqrt{p_{\max}});
    \\\notag p_{\max}\!\!=&\underset{\forall t,n\in \mathcal{N}}{\max}\;p_{t,n};\;
    B=6\alpha\Gamma+8(E-1)^2G^2;\\\notag 
     X^{\rm BF}&(G,\nu_2,\nu_{\infty})=\underset{\forall t}{\max}\; X^{\rm BF}_t(G,\nu_2,\nu_{\infty}).%
\end{align}%
\end{theo}

\begin{proof}
    See \textbf{Appendix~\ref{Appendix:convergence_proof}}.
\end{proof}

 When the bit-flipping noise (resulting from both the artificial noise and noisy wireless channels) is small enough, \ie $p_{\max}\rightarrow0$ and $X^{\rm BF}\rightarrow 0$, the convergence upper bound in~\eqref{eq:bound_T} tends to 
$\mathbb{E}\!\left(\!\Vert\bar{\mathbf{w}}_T\!-\!\mathbf{w}^*\Vert_2^2\right) \leq(1\!-\!\eta\mu)^{T}\Vert\bar{\mathbf{w}}_{0}\!-\!\mathbf{w}^*\Vert_2^2+\eta^{2}\!B\frac{1-\!(1\!-\!\eta\mu)^{T}}{\eta\mu}$,
which is consistent with the upper bound of the existing baseline developed with no privacy consideration in~\cite{Li2020On}. The validity of~\eqref{eq:bound_T} is cross-verified. 
In addition, $ X^{\rm BF}(G,\nu_2,\nu_{\infty})$ can be replaced by $ X^{\rm Gauss}(G)=\underset{\forall t}{\max}\; X^{\rm Gauss}_t(G)$ in~\eqref{eq:bound_T} under the Gaussian mechanism.

Moreover, the first term on the RHS of~\eqref{eq:bound_T} shows the dominant impact of the increasing number of training iterations, which decreases and converges to zero with the increase of $T$.
The second term accounts for the data heterogeneity of the participating devices. The third term of~\eqref{eq:bound_T} captures the impact of the bit-flipping noise on the convergence and tends to zero when no DP noise is considered.
All terms on the RHS of~\eqref{eq:bound_T} consist of the maximum noise level $p_{\max}$, which penalizes the utility of WFL. Particularly, the first term may fail to converge when $(1+\sqrt{p_{\max}})(1-\eta\mu)^{E}>1$, resulting from a reduction of the privacy budget and further enlarging the end-to-end BER. The other terms on the RHS of~\eqref{eq:bound_T} increase with $p_{\max}$. 
To this end, it is imperative to control the increase of $p_{\max}$ to balance the privacy and utility of FL.

\begin{remark}[Impact of clipping threshold]
Selecting an appropriate gradient clipping threshold $G$ is crucial. If $G$ is too small, the WFL convergence can significantly slow down~\cite{9252950,pmlr-v202-koloskova23a}. 
If $G$ is too large, the $\ell_{\infty}$-norm of the clipped gradient vector $\nu_{\infty}$ would increase the convergence upper bound and penalize the accuracy of WFL. This is due to the fact that  $X^{\rm BF}(G,\nu_2,\nu_{\infty})$ in~\eqref{eq:bound_T} is positively related to both $\nu_2$ and $\nu_\infty$.
Typically, $G$ is specified empirically.
With reference to~\cite{Wei2020Federated,DLDP}, in this paper, we set $G$ as the median of the $\ell_2$-norms of the unclipped gradient obtained during training. 
\end{remark}

\section{Experiments and Evaluations}\label{section:results}
In this section, we conduct extensive experiments to validate the convergence analysis of the proposed channel-native bit-flipping mechanism for privacy-preserving WFL.

\subsection{Estimation of Bit-Level Distance}\label{section:binary_sensitivity}

It is non-straightforward to analytically derive the proposed expected bit-level distance ${{\bar\kappa}_{\Delta\boldsymbol{\omega}_{\max} }}$ in~\eqref{binary_sensitivity}, due to the lack of the \textit{a-priori} knowledge about the continuous domain $\mathcal{R}$. Nevertheless, each device can readily estimate its expected bit-level distance ${{\bar\kappa}_{\Delta\boldsymbol{\omega}_{\max} }}$ within a practical discrete domain sampled from the continuous domain~$\mathcal{R}$.
Specifically, each device can pre-train its local model parameters and collect a set of model parameters, denoted as $\mathcal{W}$. 
It can also create a set of  $\mathcal{X}=\{\boldsymbol{x}\;|\;\Vert\boldsymbol{x}\Vert_2=\Delta \boldsymbol{\omega}_{\max} \}$ comprising a large number of samples (\eg  $10,000$ samples in our experiments) randomly selected from the $M$-dimensional vector space satisfying $\Vert\boldsymbol{x}\Vert_2=\Delta \boldsymbol{\omega}_{\max}$.
Then, the device can estimate the expected bit-level distance ${{\bar\kappa}_{\Delta\boldsymbol{\omega}_{\max} }}$, as given by 
\begin{align}
&{\bar\kappa}_{\Delta\boldsymbol{\omega}_{\max} }\!=\!\frac{\sum _{\boldsymbol{x}\in\mathcal{X},\boldsymbol{\omega}\in\mathcal{W}}\kappa(\omega_{1},\!\cdots\!,\omega_{M},x_{1},\cdots\!,x_{M})}{\sum _{\boldsymbol{x}\in\mathcal{X},\boldsymbol{\omega}\in\mathcal{W}}1},\label{estimate_binary_sensitivity}
\end{align}
which is a discrete approximation of \eqref{binary_sensitivity}.

Fig.~\ref{Different_bits} evaluates the impacts of the model size $M$ and the classical sensitivity $\Delta \boldsymbol{\omega}_{\max}$ on the estimated expected bit-level distance ${{\bar\kappa}_{\Delta\boldsymbol{\omega}_{\max} }}$. 
In general, ${{\bar\kappa}_{\Delta\boldsymbol{\omega}_{\max} }}$ grows with $M$ and $\Delta \boldsymbol{\omega}_{\max}$, while exhibiting some fluctuations. This is because, given $\Delta \boldsymbol{\omega}_{\max}$, enlarging $M$ increases the $\ell_1$-norms of the distances between bitstreams. Then, ${{\bar\kappa}_{\Delta\boldsymbol{\omega}_{\max} }}$ also increases, since ${{\bar\kappa}_{\Delta\boldsymbol{\omega}_{\max} }}$ is no smaller than the expectation of the $\ell_1$-norms; see \textbf{Remark~\ref{remark kappa}}.

\begin{figure} 
    \centering
    \includegraphics[scale=0.5]{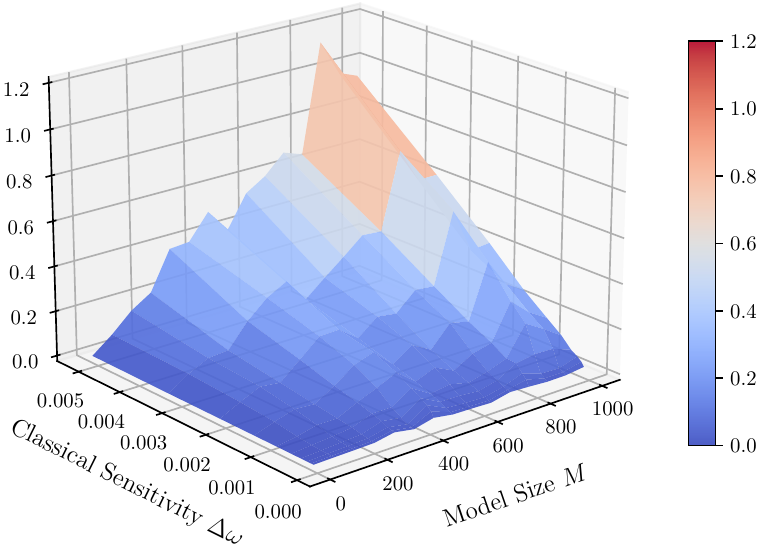}
    \caption{Estimated expected bit-level distance ${{\bar\kappa}_{\Delta\boldsymbol{\omega}_{\max} }}$ versus model parameter size $M$ and classical sensitivity $\Delta \boldsymbol{\omega}_{\max}$.}
    \label{Different_bits}
\end{figure}

\subsection{Evaluation of Bit-Flipping DP for WFL}
 
By default, we set the communication BER, $p_{t,n,\mathrm{C}}$, to be randomly and uniformly distributed within $[0,0.02]$. 
When the bit-flipping DP mechanism is in use, the end-to-end and the artificial BERs are generated using \textbf{Algorithm 1}.

\subsubsection{ML Model and Dataset}
We consider the following two ML models and two widely used public datasets. 

\vspace{2 mm}
\noindent $\bullet$ \textbf{CNN \& F-MNIST:}
We consider an image classification task to classify a non-identical and independently distributed (non-i.i.d.) F-MNIST dataset based on a CNN model. The CNN model has two convolutional layers with 32 or 64 convolutional filters per layer. The size of the CNN model is $M=1.21\times 10^6$. The F-MNIST dataset 
has $671,585$ training examples, $77,483$ testing examples and $62$ labels. There are $20$ devices, 
each randomly selecting $|\mathcal{D}_n|=2,000$ training examples with 10 labels. 
The local learning rate is $\eta=0.1$. By default, the gradient clipping threshold is $G=1$. We can empirically obtain $\nu_2=16$ and $\nu_{\infty}=0.5$ by evaluating the $\ell_{2}$- and $\ell_{\infty}$-norms of the updated model parameters. With reference to~\cite{Wei2020Federated,DLDP}, we set $G$ as the median of the norms of the unclipped gradient throughout training. The classical sensitivity is $\Delta \boldsymbol{\omega}_{\max}=10^{-4}$. The expected bit-level distance is estimated to be ${{\bar\kappa}_{\Delta\boldsymbol{\omega}_{\max} }}=0.02$, as described in Section~\ref{section:binary_sensitivity}.

\vspace{2 mm}
\noindent $\bullet$ \textbf{ResNet18 \& Fed-CIFAR100:}
Another image classification task is to train the ResNet18 model on the Fed-CIFAR100 dataset. The ResNet18 model has 18 layers, including 17 convolutional layers and a fully connected layer. Compared to a plain CNN model, the ResNet model inserts shortcut connections to prevent vanishing gradients~\cite{He_2016_CVPR}. The size of the ResNet18 model is $M=11.69\times10^6$. 
Fed-CIFAR100 divides the standard CIFAR100 into 10 training subsets with 10 image classes per subset. Each image has $32\times 32$ pixels. There are $10$ devices, each training the image samples of all 10 classes, a total of $|\mathcal{D}_n|=2,000$ images.
The learning rate is~$\eta=1$. 
The default clipping threshold is $G=1$, and correspondingly, $\nu_2=128$ and $\nu_{\infty}=1$.
The classical sensitivity is $\Delta\boldsymbol{\omega}_{\max}= 10^{-3}$. Then, the expected bit-level distance is estimated to be ${{\bar\kappa}_{\Delta\boldsymbol{\omega}_{\max} }}=0.025$.

\subsubsection{Benchmarks}
We also implement the following channel-agnostic baselines to serve as the benchmarks. 

\vspace{2 mm}
\noindent $\bullet$  \textbf{Channel-agnostic bit-flipping mechanism:}  
The artificial bit-flipping noise generated by $(\lambda,\epsilon)$-Rényi DP (see Sections~\ref{subsection_map} and~\ref{artificial_noise}) is employed solely for privacy protection, \ie $p_{t,n,\mathrm{A}}=p_{t,n}$. 
It does not exploit the communication noise for privacy protection. 

\vspace{2 mm}
    
\noindent $\bullet$ \textbf{Channel-agnostic Gaussian mechanism (with erroneous packets accepted):} 
The local models are solely protected by artificial Gaussian noise satisfying $(\epsilon,\delta)$-DP~\cite{Wei2020Federated}.  
The communication noise may flip the sign and exponent bits of \textbf{binary32} numbers, compromising the convergence of WFL.

\vspace{2 mm}
\noindent $\bullet$ \textbf{Channel-agnostic Gaussian mechanism with erroneous packets dropped:} 
    Different from the channel-agnostic Gaussian mechanism, the local models of a device are segmented and packetized. The standard WiFi packet of $2,312$ bytes is considered. Each packet is encoded with a circular redundancy check, and discarded at the server if it fails to pass the check. This mechanism represents the state of the art.

\vspace{2 mm}
In the baselines using the Gaussian mechanism, the DP noise scale is $\sigma=\frac{\Delta \boldsymbol{\omega}_{\max}K\sqrt{2\ln(1.25/\delta)}}{\epsilon'}$ satisfying $(\epsilon',\delta)$-DP~\cite{dwork2014algorithmic}.
For a fair comparison, $(\lambda,\epsilon)$-Rényi DP can be equivalently converted into $(\epsilon',\delta)$-DP,
with $\delta=\frac{e^{(\lambda-1)(\epsilon-\epsilon')}}{\lambda-1}(1-\frac{1}{\lambda})^{\lambda}$~\cite{canonne2021discrete}. In this paper, we set $\epsilon=\epsilon'$ by default.

\subsubsection{Comparison with Channel-Agnostic DP Mechanisms}

\begin{figure} 
    \centering
    \includegraphics[scale=0.4]{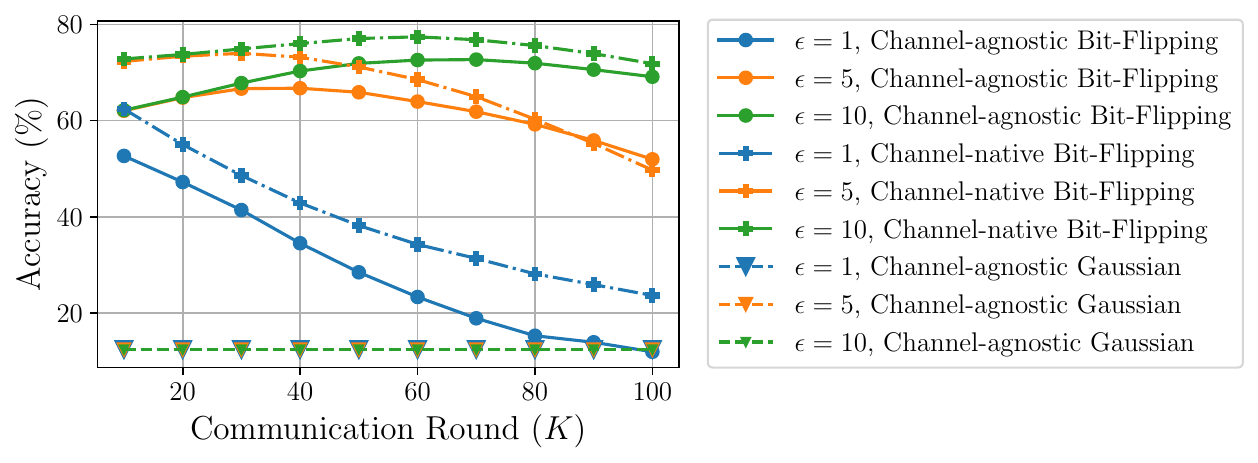}
    \caption{Comparison of the proposed channel-native mechanism with channel-agnostic mechanisms suffering from bit-flipping communication noises. ${{\bar\kappa}_{\Delta\boldsymbol{\omega}_{\max} }}=0.02$; $\Delta \boldsymbol{\omega}_{\max}=10^{-4}$; $E=50$; and $\lambda=2$. }
    \label{fig:compare_different_mechanism}
\end{figure}

Fig.~\ref{fig:compare_different_mechanism} evaluates the impact of different DP mechanisms on the accuracy of WFL in a noisy wireless environment. 
The proposed channel-native bit-flipping mechanism outperforms all channel-agnostic mechanisms, since it takes advantage of the noisy nature of wireless channels for privacy protection.
By contrast, the channel-agnostic schemes overlook the privacy-preserving capability offered by wireless channels, and hence, suffer from excessive privacy protection at the cost of accuracy.

The proposed bit-flipping mechanism outperforms the Gaussian mechanisms, even when it overlooks the privacy-preserving capability of noisy wireless channels, thanks to the new floating-point-to-fixed-point conversion method.
This is because all bits in \textbf{binary32} are transmitted in the channel-agnostic Gaussian mechanism, and their sign and exponent bits can have errors and dramatically distort the local models. By contrast, only the fraction bits are transmitted under the proposed approach, with 28\% lower traffic than the channel-agnostic Gaussian mechanism, benefiting from the floating-point-to-fixed-point conversion; see Section IV-A.
 \begin{figure} 
    \centering
    \includegraphics[scale=0.4]{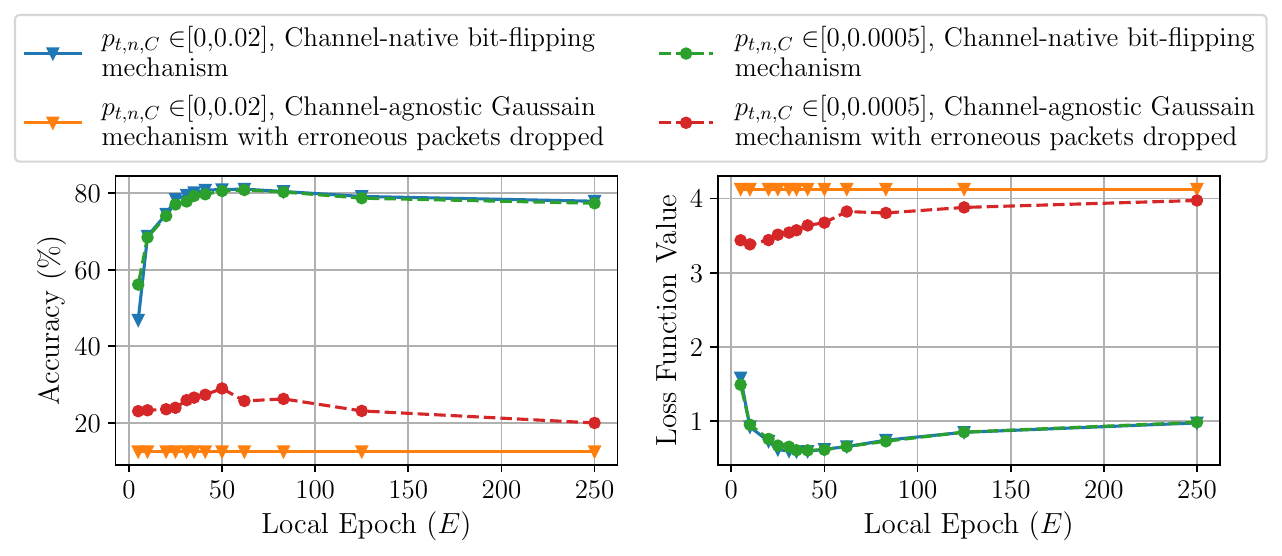}
    \caption{Comparison of the proposed bit-flipping mechanism with channel-agnostic Gaussian mechanism suffering from erroneous packets dropped. $\epsilon=10$; $T=2,500$; and $\lambda=2$.  }
    \label{fig:Packet_MNIST_CNN_accuracy_T_2500}
\end{figure}

Fig.~\ref{fig:Packet_MNIST_CNN_accuracy_T_2500} compares the proposed channel-native bit-flipping mechanism with the channel-agnostic Gaussian mechanism under a more classical setting with any erroneous packets dropped at the server. The standard WiFi packet length of $2,312$ bytes is considered. 
Both mechanisms experience the same channel conditions with the same communication BER, $p_{t,n,C}$, distributed uniformly randomly within the default range $[0, 0.02]$ and within the range $[0, 0.0005]$.
The proposed mechanism is significantly better, indicating that even erroneous local models corrupted by channel fading and noise can still contain useful model information and can be aggregated to improve the convergence and accuracy of WFL.

\subsubsection{Impact of Parameters}

 \begin{figure} 
    \centering
    \includegraphics[scale=0.4]{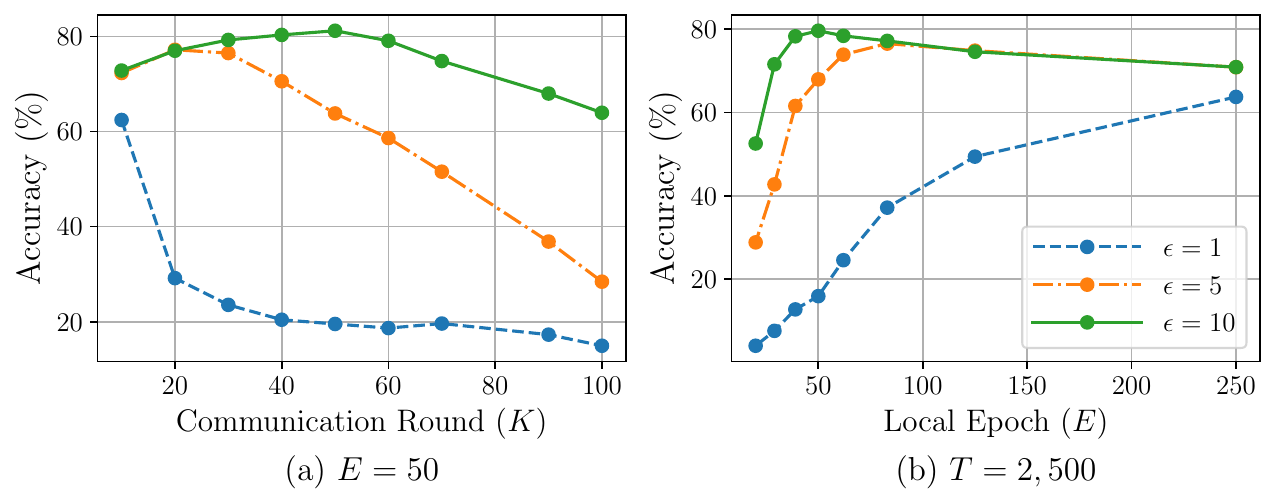}
    \caption{Accuracy of CNN on F-MNIST versus $K$ and $E$. $\lambda=2$.}
    \label{fig:MNIST_CNN_accuracy_T_2500}
\end{figure}

Fig.~\ref{fig:MNIST_CNN_accuracy_T_2500} plots the accuracy of the CNN model on the F-MNIST dataset. Fig.~\ref{fig:CIFAR100_Resnet18_accuracy_T_1500} plots the accuracy of ResNet18 on the Fed-CIFAR100 dataset. 
It is demonstrated in both figures that increasing the privacy budget $\epsilon$ leads to a decrease in the end-to-end BER, thereby improving the accuracy of WFL under the mechanism. 

 \begin{figure} [t]
    \centering
    \includegraphics[scale=0.4]{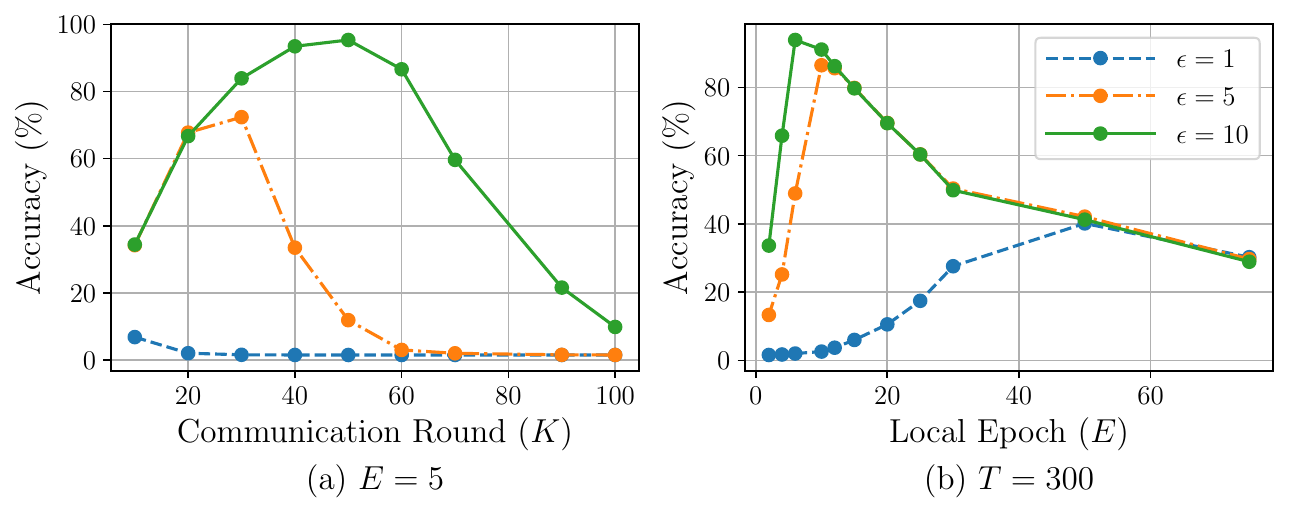}
    \caption{Accuracy of ResNet18 model on Fed-CIFAR100 dataset versus $K$, $E$ and $\epsilon$, where $\lambda=2$.}
    \label{fig:CIFAR100_Resnet18_accuracy_T_1500}
\end{figure}

Figs.~\ref{fig:MNIST_CNN_accuracy_T_2500}(a) and~\ref{fig:CIFAR100_Resnet18_accuracy_T_1500}(a) examine the impact of the number of local epochs $E$ on the utility of two image classification tasks, respectively.
It is observed that the optimal value of $E$ exists within the given range. This is because a larger $E$ reduces the number of communication rounds, \ie $K$, when the total number of iterations, $T$, is given, thereby alleviating the impact of the bit-flipping noise on the training accuracy of WFL under the proposed bit-flipping DP mechanism. On the other hand, in the face of the heterogeneity of the local datasets, a larger $E$ can cause the local models to be biased towards the local data of the clients, resulting in poor global training performance.

Figs.~\ref{fig:MNIST_CNN_accuracy_T_2500}(b) and~\ref{fig:CIFAR100_Resnet18_accuracy_T_1500}(b) evaluate the proposed DP mechanism, where $T$ varies while $E$ remains unchanged.
The efficiency of the image classification tasks initially increases and then decreases, as $T$ increases. The underlying reason is two-fold. On the one hand, a decrease in $T$ leads to a decrease in the end-to-end BER, $p_{t,n}$, benefiting the convergence of WFL. On the other hand, an increase of $T$ allows for more local iterations and global aggregations, improving the accuracy of WFL. In this sense, the optimal $T$ exists in the proposed mechanism.

As $T$ keeps increasing, the end-to-end BER, $p_{t,n}$, increases and converges to $0.5$. In this case, the local models are severely corrupted by the artificial and communication noises and become nearly completely random.
As a consequence, WFL training cannot even start properly.

\subsection{Statistical Properties of Bit-Flipping}\label{section:results2}

 Fig.~\ref{fig:mean and Variance} shows the consistency between the analytical and empirical results of the mean and variance of~$\tilde{a}$. The validity of \textbf{Lemma~\ref{lemma:expectation and Var}} is confirmed. 
  As shown in Fig.~\ref{fig:mean and Variance}(a), the mean of~$\tilde{a}$ is linear in the range of the original floating-point numbers with the same sign and exponent parts, which is consistent with~\eqref{expectation_of_a}. As shown in Fig.~\ref{fig:mean and Variance}(b), the variance of $\tilde{a}$ does not change with $a$ in the range with the same sign and exponent parts, which is consistent with~\eqref{Var_of_a}.
  These phenomena can be observed across $a\in[-\infty,\infty]$.

\begin{figure}
    \centering
    \begin{subfigure}[b]{0.24\textwidth}
        \includegraphics[width=4.3cm]{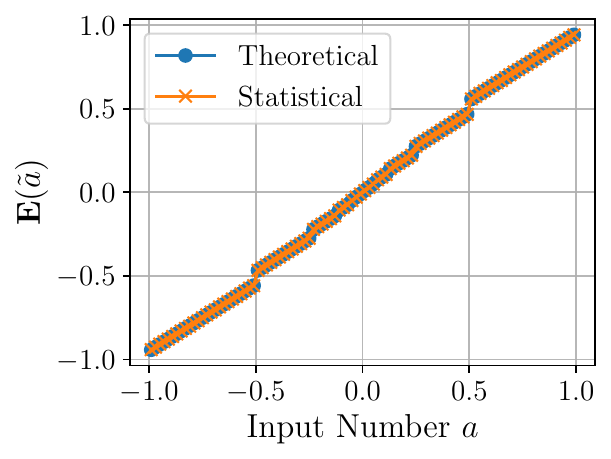}
        \label{fig:sub1}
    \end{subfigure}
    \hfill
    \begin{subfigure}[b]{0.24\textwidth}
        \includegraphics[width=4.3cm]{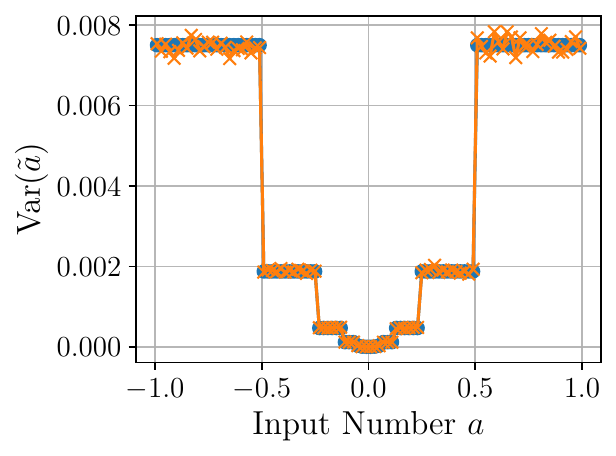}
        \label{fig:sub2}
    \end{subfigure}
    \caption{The statistical properties of perturbed values $\tilde{a}$ vs. original numbers $a$, where the bit-flipping probability is $p=0.1$ and we show $a\in [-1,+1]$ for the clarity of the figure.}\label{fig:mean and Variance}
\end{figure}

\section{Conclusion}
\label{sec:conclusion}
This paper proposed a novel channel-native bit-flipping DP mechanism that exploits inherent noise in wireless channels to preserve the privacy of WFL. 
Specifically, we proposed a new floating-point-to-fixed-point conversion method, which only requires the bits in the fraction part of model parameters to be transmitted and hence prevents dramatic model distortions that would otherwise occur when transmitting floating-point numbers, \eg IEEE~754 \textbf{binary32}.
We also analyzed the expected bit-level distance of the model parameters, and proved rigorously that the proposed mechanism satisfies $(\lambda,\epsilon)$-Rényi DP and facilitates WFL convergence. 
Extensive experimental evaluations of two image classification tasks demonstrated the superiority of the proposed mechanism compared to the state-of-the-art channel-agnostic Gaussian mechanisms. 

\appendix
\subsection{Proof of \textbf{Lemma~\ref{lemma:expectation and Var}}}\label{proof_E_Var_of_a}

We start by proving~\eqref{expectation_of_a}.
After flipping the bits in the fraction part of $a$ with probability $p$ (resulting from the bit-flipping perturbation and receiver noise), the probability of each received bit, $\tilde{a}_i,i=0,1,\ldots,22$, is given by
\begin{align}
         \Pr(\tilde{a}_i|{a}_i)=\left\{ \begin{array}{ll}
1-p, &  \text{if $a_{i}$ is received correctly;}\\
p, & \text{otherwise,}
\end{array}\right.\notag
    \end{align}
 Then, the expectation of each received bit is given by
\begin{align}\notag
         \mathbb{E}(\tilde{a}_i)=(1-p)a_i+p(1-a_i)=a_i(1-2p)+p.%
    \end{align}%

The (decimal) mean of $\tilde{a}=(\tilde{a}_{22}\tilde{a}_{21}\cdots\tilde{a}_0)_2$ is given by
    \begin{align}
      \!\!\! \!\!\mathbb{E}(\tilde{a}) \!\!=\!\!(\!-\!1)^{a_{31}}2^{(a_{30}a_{29}...a_{23})_{2}\!-\!127}\Big(1\!\!+\!\!\!{\sum}_{i\!=\!1}^{23}\mathbb{E}(\tilde{a}_{23\!-\!i})2^{\!-\!i}\Big). \label{eq:flipped expectation 1}
    \end{align}
By plugging $\mathbb{E}(\tilde{a}_i)=a_i(1-2p)+p$ into~\eqref{eq:flipped expectation 1}, $\mathbb{E}(\tilde{a})$ becomes
    \begin{small}
         \begin{align}
       \notag\mathbb{E}(\tilde{a})
      \notag =&(1\!\!-\!\!2p)a\!\!+\!\!(\!-\!1)^{a_{31}}2^{(a_{30}a_{29}...a_{23})_{2}\!-\!127}\left(2p\!+\!{\sum}_{i\!=\!1}^{23}2^{\!-\!i}p\right),%
   \end{align}%
    \end{small}%
    which leads to~\eqref{expectation_of_a}.
Next, we prove \eqref{Var_of_a}. 
After flipping the fraction part of $a$ with probability $p$, the variance of each bit is $\textrm{Var}(\tilde{a}_{i})=p(1-p),i=0,\cdots,22$.  
   Then, the (decimal) variance of the bit-flipping value is given by 
    \begin{align}
        \notag\mathrm{Var}(\tilde{a}) &=\left((-1)^{a_{31}}2^{(a_{30}a_{29}...a_{23})_{2}-127}\right)^2{\sum}_{i=1}^{23}p(1-p)2^{-2 i}\\
        &=\frac{(1\!\!-\!\!4^{\!-\!23})}{3}p(1\!\!-\!\!p)\left(2^{(a_{30}a_{29}...a_{23})_{2}\!-\!127}\right)^2,\label{eq:flipped_variance_1}
    \end{align}
which leads to~\eqref{Var_of_a}.

\subsection{Proof of \textbf{Theorem~\ref{theo1_renyi}}}\label{proof_theo1}
{ 
For conciseness, we suppress the subscripts ``\(_t\)'' and ``\(_n\)'' from \(p_{t,n}\), \(\boldsymbol{u}_{t,n}\), \(\mathcal{D}_n\), \(\Delta\boldsymbol{\omega}_{\max,n}\), and ${\bar\kappa}_{\Delta\boldsymbol{\omega}_{\max,n} }$ in the rest of this section.
The bit-flipping DP mechanism applied to the local model parameters $\boldsymbol{u}\in\mathbb{B}^{23M\times1}$ adheres to $(\lambda, \frac{\epsilon}{K})$-Rényi DP, and satisfies 
\begin{align} 
D_{\lambda}(Q_{ \boldsymbol{u}}\| Q_{\boldsymbol{u}'}) =\sum_{k=1}^{23M}  D_{\lambda}(Q_{k,u_k}[\boldsymbol{\rho}_k]\|Q_{k,u_k'}[\boldsymbol{\rho}_k])\leq \frac{\epsilon}{K},\label{D_all_definition}
 \end{align}
where the Rényi divergence $D_{\lambda}(Q_{ \boldsymbol{u}}\| Q_{\boldsymbol{u}'})$
is based on the composition of independently applying the bit-flipping mechanism to each bit of the bitstream, incorporating the corresponding significance (binary weight) $2^{(k{\rm\, mod\,}  23)-23}$ of the $k$-th bit; $Q_{ \boldsymbol{u}}$ and $Q_{ \boldsymbol{u}'}$ denote the output distribution after applying bit-flipping mechanism to $\boldsymbol{u}$ and $\boldsymbol{u}'$, respectively;
$D_{\lambda}(Q_{k,u_k}[\boldsymbol{\rho}_k]\|Q_{k,u_k'}[\boldsymbol{\rho}_k])$ is the Rényi divergence of $k$-th bit:
{\small\begin{align}
\label{define_divergence_lambda_L} D_{\lambda}(Q_{k,u_k}[\boldsymbol{\rho}_k]\|Q_{k,u_k'}[\boldsymbol{\rho}_k])\!\triangleq\!\frac{1}{\lambda\!-\!1}\ln\bigg[\mathbb{E}_{\boldsymbol{\rho}_{k}\sim Q_{k,u_k'}}\Big(\frac{Q_{k,u_k}[\boldsymbol{\rho}_k]}{Q_{k,u_k'}[\boldsymbol{\rho}_k]}\Big)^{\lambda}\bigg],
\end{align}%
}%
where $\boldsymbol{\rho}_{k}$ denotes the resultant decimal representation after flipping the $k$-th bit $u_k$ (or $u_k'$) in the received bitstream, i.e.,
$\boldsymbol{\rho}_{k} = \sum_{i=0}^{22} 2^{i-23} \rho_i$, and 
$\rho_{(k\, \mathrm{mod}\, 23)} \triangleq \hat{u}_k$ (or $\hat{u}_k'$) denotes the received version of $u_k$ (or $u_k'$), which may have been flipped. The remaining bits $\rho_i$, $\forall i \neq (k\, \mathrm{mod}\, 23)$, are assumed to be independently perturbed with unknown distributions.
$Q_{k,u_k}[\boldsymbol{\rho}_k]$ and $Q_{k,u_k'}[\boldsymbol{\rho}_k]$ are the probability mass functions (PMFs) of the decimal representation $\boldsymbol{\rho}_k$ after randomly flipping $\hat{u}_k$ and $\hat{u}_k'$, respectively, each weighted by $2^{(k\,{\rm mod}\, 23)-23}$, assuming all other bits are varied independently and treated as unknown.
$D_{\lambda}(\cdot||\cdot)$ denote Rényi divergence of order $\lambda$. $\mathbb{E}_{\boldsymbol{\rho}_{k}\sim Q_{k,u_k'}}\big(\frac{Q_{k,u_k}[\boldsymbol{\rho}_k]}{Q_{k,u_k'}[\boldsymbol{\rho}_k]}\big)^{\lambda}$ is the key term of Rényi divergence.

Suppose that the bit-level difference between ${u}_k$ and ${u}_k'$ attains its maximum value (i.e., $2^{(k{\,\rm \,mod} 23)-23}$) with probability $q_k$, or zero with probability $1 - q_k$. Then, the expectation of its bit-level distance between ${u}_k$ and ${u}_k'$ is $2^{(k{\,\rm \,mod} 23)-23}q_k$.

Now, we can introduce an auxiliary bitstream $\boldsymbol{b}\in\mathbb{B}^{23M\times1}$. The $k$-th bit of its fraction part, $b_k$, differs from $u_k'$ with the maximum bit-level distance, with probability $1$. 
Let $Q_{k,b_k}[\boldsymbol{\rho}_k]$ denote the PMFs of the decimal representations $\boldsymbol{\rho}_k$ after undergoing the bit-flipping operation on \( b_k \).
Based on the auxiliary distribution $Q_{k,b_k}[\boldsymbol{\rho}_k]$ and $Q_{k,u_k'}[\boldsymbol{\rho}_k]$, the output of $Q_{k,u_k}[\boldsymbol{\rho}_k]$ is a mixed distribution:
\begin{align}
    Q_{k,u_k}[\boldsymbol{\rho}_k]= q_kQ_{k,b_k}[\boldsymbol{\rho}_k] + (1-q_k)Q_{k,u_k'}[\boldsymbol{\rho}_k].\label{eq:mix_distribution}
\end{align}

Next, we factor in the weight \(2^{(k\,{\rm mod}\, 23)-23}\) corresponding to the $k$-th fraction bit. 
Since only the fraction parts are encoded into the bitstream, $u_k'$ and $b_k$ correspond to the $m$-th decimal number ($m=\left\lfloor \frac{k}{23} \right\rfloor$) of the model parameter vector, i.e., 
$\boldsymbol{u}_{m}' = \sum_{i=0}^{22} u_{i+23m}' 2^{i-23}$ and $\boldsymbol{b}_{m} = \sum_{i=0}^{22} b_{i+23m} 2^{i-23}$. 
Clearly, both $\boldsymbol{u}_{m}'$ and $\boldsymbol{b}_{m}$ lie within the range $[0,1)$.

Then, we evaluate the probability of distinguishing between the pair $(\boldsymbol{u}_{m}', \boldsymbol{b}_{m})$ when only $u_k'$ and $b_k$ are independently considered. Since the other bits of $\boldsymbol{u}_{m}'$ and $\boldsymbol{b}_{m}$ are independent and unknown, the resolution with which the real number can be identified is $2^{(k{\,\rm mod \,} 23) - 23}$. We divide the range $[0,1)$ into $2^{23-(k{\,\rm mod \,} 23)}$ intervals. Each interval collects the numbers whose fraction parts share identical bits from the $22$nd bit down to the $(k{\,\rm mod \,} 23)$-th bit, while the bits from position $(k{\,\rm mod \,} 23)-1$ down to $0$ may vary and its distribution are unknown. 
In particular, $\boldsymbol{u}_{m}'$ and $\boldsymbol{b}_{m}$ correspond to two distinct intervals, $\boldsymbol{X}_{\boldsymbol{u}_{m}'}$ and $\boldsymbol{X}_{\boldsymbol{b}_{m}} $, respectively.
The other intervals indicate the distribution of the other bits ($i\neq k{\rm \,mod\,} 23 $) being flipped or not, which does not contribute to distinguishing between the pair $(\boldsymbol{u}_{m}', \boldsymbol{b}_{m})$.

Without loss of generality, we analyze the output distribution of the decimal representation of $\hat{u}_k'$ after applying the bit-flipping operation to $u_k'$ with probability $p$.
 \begin{itemize}
     \item The probability that the received value $\hat{u}_k'$ equals $b_k$ is $p$. 
Since the other bits are governed by independent mechanisms and treated as unknown, the distribution of the decimal representation of $\hat{u}_k'$ must rely on random guessing for the remaining bits. 
Consequently, the probability that the interval $\boldsymbol{X}_{\boldsymbol{b}_{m}}$, corresponding to the decimal representation $\boldsymbol{\rho}_k$, aligns with that of $\boldsymbol{b}_m$ is $2^{(k\,{\rm mod}\, 23)-23} \cdot p$.
\item The received value $\hat{u}_k'$ remains $\hat{u}_k'$ with probability $1 - p$. 
Consequently, the probability that the interval $\boldsymbol{X}_{\boldsymbol{u}_{m}'}$, corresponding to the decimal representation $\boldsymbol{\rho}_k$ of $\hat{u}_k'$, aligns with that of $\boldsymbol{u}_m'$ is $2^{(k\,{\rm mod}\, 23)-23} \cdot (1 - p)$.
\item 
When the decimal representation of $\hat{u}_k'$ falls into intervals other than $\boldsymbol{X}_{\boldsymbol{b}_{m}}$ and $\boldsymbol{X}_{\boldsymbol{u}_{m}'}$, it does not help distinguish between the pair $(\boldsymbol{u}_{m}', \boldsymbol{b}_{m})$. 
The probability of any of these intervals is $2^{(k\,{\rm mod}\, 23) - 23}$, and the total probability of falling into these indistinguishable intervals is $1 -  2^{(k\,{\rm mod}\, 23) - 23}$.%
 \end{itemize}%
By summarizing the above three cases, the overall distribution of $\boldsymbol{\rho}_k \sim Q_{k,u_k'}[\boldsymbol{\rho}_k]$ is given by
\begin{align}
&Q_{k,u_k'}[\boldsymbol{\rho}_k] =\notag\\
&\quad\!\!\!\!\!
\begin{cases}
\Pr(\boldsymbol{\rho}_k \in \boldsymbol{X}_{\boldsymbol{b}_{m}}) = 2^{(k\,{\rm mod}\, 23)-23} \cdot p; \\
\Pr(\boldsymbol{\rho}_k \in \boldsymbol{X}_{\boldsymbol{u}_{m}'}) = 2^{(k\,{\rm mod}\, 23)-23} \cdot (1 - p); \\
\Pr(\boldsymbol{\rho}_k \in [0,1) \!\setminus \!\boldsymbol{X}_{\boldsymbol{b}_{m}} \!\setminus\! \boldsymbol{X}_{\boldsymbol{u}_{m}'}) \!\!= \!1 \!\!- \!2^{(k\,{\rm mod}\, 23)-23}.
\end{cases}
\label{uk'_distribution}
\end{align}
By following the same analysis steps, the distribution of $Q_{k,b_k}[\boldsymbol{\rho}_k]$ is given by
\begin{align}
&Q_{k,b_k}[\boldsymbol{\rho}_k] =\notag\\
&\quad\!\!\!\!\!
\begin{cases}
\Pr(\boldsymbol{\rho}_k \in \boldsymbol{X}_{\boldsymbol{b}_{m}}) \!\!= \!\!2^{(k\,{\rm mod}\, 23)\!-\!23} \cdot (1 \!\!-\!\! p); \\
\Pr(\boldsymbol{\rho}_k \in \boldsymbol{X}_{\boldsymbol{u}_{m}'}) \!\!= \!\!2^{(k\,{\rm mod}\, 23)\!-\!23} \cdot p; \\
\Pr(\boldsymbol{\rho}_k \in [0,1) \!\setminus \!\boldsymbol{X}_{\boldsymbol{b}_{m}} \!\setminus\! \boldsymbol{X}_{\boldsymbol{u}_{m}'}) \!\!=\!\! 1 \!-\! 2^{(k\,{\rm mod}\, 23)-23}.
\end{cases}
\label{bk_distribution}
\end{align}
By substituting \eqref{uk'_distribution} and \eqref{bk_distribution} into \eqref{eq:mix_distribution}, it follows that
\begin{align}
&\!\!\!Q_{k,u_k}[\boldsymbol{\rho}_k] =\notag
\\
&\!\begin{cases}
\Pr(\boldsymbol{\rho}_k \!\in\! \boldsymbol{X}_{\boldsymbol{b}_{m}}\!)\!\! =\!\! 2^{(k\,{\rm mod}\, 23)\!-23}  [p(1 \!\!-\!\! q_k) \!\!+\!\! (1 \!\!-\! p)q_k]; \\\!
\Pr(\boldsymbol{\rho}_k \!\in\! \boldsymbol{X}_{\boldsymbol{u}_{m}'}\!)\! \!=\!\! 2^{(k\,{\rm mod}\, 23)\!-23}  [(1 \!\!-\!\! p)(1 \!\!-\!\! q_k) \!\!+\!\! p q_k]; \\\!
\Pr(\boldsymbol{\rho}_k \in [0,1) \!\setminus\! \boldsymbol{X}_{\boldsymbol{b}_{m}} \!\setminus \!\boldsymbol{X}_{\boldsymbol{u}_{m}'}) \!= \!1\! -\! 2^{(k\,{\rm mod}\, 23)-23}.
\end{cases}
\label{uk_distribution}
\end{align}
Then, we can rewrite $\mathbb{E}_{\boldsymbol{\rho}_{k}\sim Q_{k,u_k'}}\Big(\frac{Q_{k,u_k}[\boldsymbol{\rho}_k]}{Q_{k,u_k'}[\boldsymbol{\rho}_k]}\Big)^{\lambda}$ as
{\small
\setlength{\abovedisplayskip}{3pt}
\setlength{\belowdisplayskip}{3pt}
\begin{subequations}
\begin{align}
&\notag\mathbb{E}_{\boldsymbol{\rho}_{k}\sim Q_{k,u_k'}}\Big(\frac{Q_{k,u_k}[\boldsymbol{\rho}_k]}{Q_{k,u_k'}[\boldsymbol{\rho}_k]}\Big)^{\lambda}
\\=&\Pr(\boldsymbol{\rho}_{k}\!\in\!\boldsymbol{X}_{\boldsymbol{b}_{m}})\Big(\frac{Q_{k,u_k}[\boldsymbol{\rho}_k]}{Q_{k,u_k'}[\boldsymbol{\rho}_k]}\Big)^{\lambda} \!+\!\Pr(\boldsymbol{\rho}_{k}\!\in\!\boldsymbol{X}_{\boldsymbol{u}_{m}'})\Big(\frac{Q_{k,u_k}[\boldsymbol{\rho}_k]}{Q_{k,u_k'}[\boldsymbol{\rho}_k]}\Big)^{\lambda} \notag\\&+\Pr(\boldsymbol{\rho}_{k}\in[0,1)\setminus\boldsymbol{X}_{\boldsymbol{b}_{m}}\setminus\boldsymbol{X}_{\boldsymbol{u}_{m}'})\Big(\frac{Q_{k,u_k}[\boldsymbol{\rho}_k]}{Q_{k,u_k'}[\boldsymbol{\rho}_k]}\Big)^{\lambda} \label{weightonly_a}
\\
=&(1-2^{(k\,\mathrm{mod}\, 23)-23}) \!+\! 2^{(k\,\mathrm{mod}\, 23)-23}p \cdot ((1\! -\! q_k) + q_k  \tfrac{1 \!-\! p}{p})^{\lambda} \notag\\
&+ 2^{(k\,\mathrm{mod}\, 23)-23}(1-p) \cdot ((1 - q_k) + q_k  \tfrac{p}{1 - p})^{\lambda} \label{eq:ratio_b}
\\
\leq& (1-2^{(k\,\mathrm{mod}\, 23)-23}) + 2^{(k\,\mathrm{mod}\, 23)-23}((1-q_k) + q_k \cdot \tfrac{(1-p)^{\lambda-1}}{p^{\lambda-1}}) \label{eq:ratio_c}
\\
=& 1 - 2^{(k\,\mathrm{mod}\, 23)-23} q_k \left(\tfrac{(1-p)^{\lambda-1}}{p^{\lambda-1}} - 1\right), \label{eq:ratio_d}%
\end{align}%
\end{subequations}%
}%
where \eqref{weightonly_a} is the expansion of $\mathbb{E}_{\boldsymbol{\rho}_{k}\sim Q_{k,u_k'}}\big(\frac{Q_{k,u_k}[\boldsymbol{\rho}_k]}{Q_{k,u_k'}[\boldsymbol{\rho}_k]}\big)^{\lambda}$, \eqref{eq:ratio_b} is obtain by substituting \eqref{uk_distribution} and \eqref{uk'_distribution} in \eqref{weightonly_a}, and \eqref{eq:ratio_c} is due to the following inequality, 
\begin{align}
    p\!\cdot\!\big((1\!-\!q_k)\!&+\!q_k\!\cdot\!\tfrac{1\!-\!p}{p}\big)^{\lambda}
    +(1\!-\!p)\!\cdot\!\big((1\!-\!q_k)\!+\!q_k\!\cdot\!\tfrac{p}{1\!-\!p}\big)^{\lambda} \notag\\
    &\leq (1\!-\!q_k)\!+\!q_k\!\cdot\!\tfrac{(1\!-\!p)^{\lambda-1}}{p^{\lambda-1}}, \label{ieq1}
\end{align}
which holds under the conditions \( 0 < q_k < 1 \), \( 0 < p < 0.5 \), and \( \lambda > 1 \) and can be verified numerically. 

By substituting \eqref{eq:ratio_d} into \eqref{define_divergence_lambda_L}, the Rényi divergence can be upper-bounded by
\begin{align}
    D_{\lambda}\!\left(Q_{k,u_k}[\boldsymbol{\rho}_k]\,\|\,Q_{k,u_k'}[\boldsymbol{\rho}_k]\right)
    \leq \tfrac{q_k\,2^{(k\,\mathrm{mod}\,23)-23}}{\lambda\!-\!1}
    \left(\tfrac{(1\!-\!p)^{\lambda\!-\!1}}{p^{\lambda\!-\!1}}\!-\!1\right).\label{D_k}
\end{align}
Substituting \eqref{D_k} into \eqref{D_all_definition}, the upper bound on the divergence for the entire bitstream over all bits \(\forall k = 1, \cdots, 23M\) can be rewritten as  
\begin{align}
    D_{\lambda}\!\left(Q_{\boldsymbol{u}} \!\| Q_{\boldsymbol{u}'}\right)
    &\leq \tfrac{1}{\lambda\!-\!1} {\sum}_{k=1}^{23M} q_k\,2^{(k\,\mathrm{mod}\,23)\!-\!23}
    \left(\tfrac{(1\!-\!p)^{\lambda\!-\!1}}{p^{\lambda\!-\!1}}\!-\!1\right) \notag \\
    &= \tfrac{{\bar\kappa}_{\Delta\boldsymbol{\omega}_{\max,n}}}{\lambda\!-\!1}
    \left(\tfrac{(1\!-\!p)^{\lambda\!-\!1}}{p^{\lambda\!-\!1}}\!-\!1\right),\notag
\end{align}
where \( \sum_{k=1}^{23M} q_k 2^{(k\,{\rm mod}\, 23)-23} = {\bar\kappa}_{\Delta\boldsymbol{\omega}_{\max,n}} \) is the expected bit-level distance, as defined in \eqref{expected_bit_level_distance}.
When 
\(
\frac{{\bar\kappa}_{\Delta\boldsymbol{\omega}_{\max,n}}}{\lambda - 1} \left( \frac{(1 - p)^{\lambda - 1}}{p^{\lambda - 1}} - 1 \right) \leq \frac{\epsilon}{K},
\)
we can directly obtain~\eqref{D_all_definition}; i.e., $(\lambda, \frac{\epsilon}{K})$-Rényi DP is satisfied in each round.
 
Further, we evaluate the end-to-end BER required to ensure that the bit-flipping mechanism satisfies $(\lambda,\epsilon)$-Rényi DP across all communication rounds, as given by
\begin{align}
\tfrac{{\bar\kappa}_{\Delta\boldsymbol{\omega}_{\max,n}}}{\lambda\!-\!1}
\left( \tfrac{(1\!-\!p)^{\lambda\!-\!1}}{p^{\lambda\!-\!1}} \!-\! 1 \right)
&\leq \tfrac{\epsilon}{K} \notag\\
\Leftrightarrow\quad
p &\geq \left(1 +
\Big( \tfrac{(\lambda\!-\!1)\epsilon}{K {\bar\kappa}_{\Delta\boldsymbol{\omega}_{\max,n}}} \right)^{\frac{1}{\lambda\!-\!1}}
\Big)^{-1}.
\label{eq:Joint_BER_preserving_privacy}
\end{align}
Note that~\eqref{eq:Joint_BER_preserving_privacy} holds under the assumption of \(p < \frac{1}{2}\).  
This completes the proof.}

\subsection{Proof of \textbf{Lemma~\ref{lemma:z}}}\label{Appendix:2}
We start with~\eqref{eq:E_z_tn}.
As discussed in Section~\ref{subsection_map}, each model element $\omega_{t,n,m}$, $m=1,\cdots,M$, 
is mapped into a fixed-point number, i.e., $\omega_{t,n,m}+3\times2^{(c_{30}c_{29}...c_{23})_{2}-126}$ within the range $[2^{(c_{30}c_{29}...c_{23})_{2}-125},2^{(c_{30}c_{29}...c_{23})_{2}-124 }]$.
After superimposing the artificial and communication bit-flipping noise, the demodulated value is still within the above range, given the error-free exponent and sign parts. 
Based on~\eqref{expectation_of_a}, the expectation of the demodulated value is
\begin{align}
\notag(1-&2p_{t,n})(\omega_{t,n,m}+3\times2^{(c_{30}c_{29}...c_{23})_{2}-126})+\\&2^{(c_{30}c_{29}\dots c_{23})_{2}-125}\Big(2p_{t,n}+{\sum}_{i=1}^{23}2^{-i}p_{t,n}\Big).
\end{align}   
After reversing the mapping operation, the expectation of the finally recovered model parameter is given by
\begin{subequations}
\begin{align}
\mathbb{E}(\tilde{\omega}_{t,n,m})
\notag& 
=(1\!\!-\!\!2p_{t,n}){\omega_{t,n,m}}\!\!-\!p_{t,n}\!\times\!2^{(c_{30}c_{29}...c_{23})_{2}\!-\!125}\!\times\! 2^{-23}\\
\!\!\!\!\!\!\!\!\!\!\!\!\!\!\!\!& \approx  (1-2p_{t,n}){\omega_{t,n,m}}.\notag
\end{align} 
\end{subequations}  
Here, $p_{t,n}\times2^{(c_{30}c_{29}...c_{23})_{2}-125}\times 2^{-23}\ll(1-2p_{t,n}){\omega}_{t,n,m}$ is negligibly smaller than $(1-2p_{t,n}){\omega}_{t,n,m}$, leading to~\eqref{eq:E_z_tn}.

Next, we prove~\eqref{eq:Var_z_tn}.
Since each element in the demodulated model parameters is within the same range $[2^{(c_{30}c_{29}...c_{23})_{2}-125}$, $2^{(c_{30}c_{29}...c_{23})_{2}-124 }]$ and the mapping operation does not affect the randomness, from~\eqref{Var_of_a}, the variance of each element is
\begin{align}
  \!\!\!\!  \mathbb{E}\left(\tilde{\omega}_{t,n,m}\!\!-\!\!\mathbb{E}(\tilde{\omega}_{t,n,m})\right)^2\!\!=&\frac{(1\!\!-\!\!4^{\!-\!23})}{3}\times \nonumber \\
  &p_{t,n}(1\!\!-\!\!p_{t,n})2^{2(c_{30}c_{29}...c_{23})_{2}\!-\!250}.\label{var_w_j}
\end{align}

Based on the definition of $\ell_2$-norm and 
\eqref{eq:z_tn}, we have
    \begin{align}
    \mathbb{E}(\Vert \boldsymbol{z}_{t,n}\!\!-\!\!\mathbb{E}(\boldsymbol{z}_{t,n})\Vert_2^2)
    &\!\!=\!\!{\sum}_{m\!=\!1}^M \mathbb{E}\left(\tilde{\omega}_{t,n,m}\!\!-\!\!\mathbb{E}(\tilde{\omega}_{t,n,m})\right)^2,\label{var_w_2}
    \end{align}
where~\eqref{var_w_2} is obtained because the random values, $\tilde{\omega}_{t,n,m},\forall m$, are independent. By plugging~\eqref{var_w_j} into~\eqref{var_w_2}, we obtain~\eqref{eq:Var_z_tn}.

Then, we prove~\eqref{eq:Ez^2}, as given by
\begin{subequations}\small\begin{align}
&\mathbb{E}( \Vert\boldsymbol{z}_{t,\mathcal{G}}\Vert_2^2)=\Vert\mathbb{E}(\boldsymbol{z}_{t,\mathcal{G}})\Vert_2^{2}\!\!+\!\!\mathbb{E}(\Vert\boldsymbol{z}_{t,\mathcal{G}}\!\!-\!\!\mathbb{E}(\boldsymbol{z}_{t,\mathcal{G}})\Vert_2^{2})\label{eq:proof_of_E_z_tG_a}
 \\&\leq{\sum}_{n\in\mathcal{N}}q_n\Vert\mathbb{E}(\boldsymbol{z}_{t,n})\Vert_2^{2}
\!\!+\!\!{\sum}_{n\in\mathcal{N}}q_n^2\mathbb{E}(\Vert \boldsymbol{z}_{t,n}\!\!-\!\!\mathbb{E}(\boldsymbol{z}_{t,n})\Vert_2^2)\label{eq:proof_of_E_z_tG_b}
\\&\leq {\sum}_{n\in\mathcal{N}}q_n p_{t,n}^2 \nu_{2}^2+\!\!\frac{(1\!\!-\!\!4^{\!-\!23})M}{3}{\sum}_{n\in\mathcal{N}}q_n^2 p_{t,n}(1\!\!-\!\!p_{t,n})\times \notag
\\
&\quad \quad\quad\quad\quad\quad\quad\quad\quad\quad\quad\quad\quad2^{2(c_{30}c_{29}...c_{23})_{2}\!-\!250},\label{eq:proof_of_E_z_tG_c}
\end{align}
\end{subequations}
where~\eqref{eq:proof_of_E_z_tG_a} is based on $\mathbb{E}(\left\Vert a\right\Vert_2 ^{2})=\mathbb{E}(\left\Vert a-\mathbb{E}[a]\right\Vert_2 ^{2})+\left\Vert \mathbb{E}[a]\right\Vert_2 ^{2}$ with $a=\boldsymbol{z}_{t,\mathcal{G}}$. 
\eqref{eq:proof_of_E_z_tG_b} is obtained by applying Cauchy's inequality to $\Vert\mathbb{E}(\boldsymbol{z}_{t,\mathcal{G}})\Vert_2^{2}$ and then exploiting the equality between the sum variance of independent variables $\{\boldsymbol{z}_{t,n}, \forall n\}$ and the variance of their sum.
\eqref{eq:proof_of_E_z_tG_c} is obtained by substituting~\eqref{eq:E_z_tn} and~\eqref{eq:Var_z_tn} into~\eqref{eq:proof_of_E_z_tG_b} and applying $|\omega_{t,n,m}|\leq \nu_{\infty}\leq2^{(c_{30}c_{29}...c_{23})_{2}-126}$, $\forall m$. \eqref{eq:proof_of_E_z_tG_c} can be rewritten as~\eqref{eq:Ez^2}.

\subsection{Proof of \textbf{Theorem~\ref{theo:convergence}}}\label{Appendix:convergence_proof}
Referring to~\cite[Eq. (11)]{Li2020On}, we have
\begin{align}
    \Vert\bar{\boldsymbol{\omega}}_t-\mathbf{w}^*\Vert_2^2\leq(1-\eta\mu) \Vert\bar{\boldsymbol{\omega}}_{t-1}-\mathbf{w}^*\Vert_2^2+\eta^2 B, \;\forall t,\label{eq:delta_omega}
\end{align}
where no bit-flipping DP noise is considered and $B=6\alpha\Gamma+8(E-1)^2G^2$ full-batch gradient descent with $\sigma_n=0,\forall n$. 

When taking the bit-flipping DP noise into account, the one-step upper bound of $\Vert\bar{\mathbf{w}}_t-\mathbf{w}^*\Vert_2^2,\forall t$ is given by
\begin{subequations}\small
    \begin{align}
    \forall t\notin\mathcal{I}_E:\,&
    \Vert\bar{\mathbf{w}}_t\!-\!\mathbf{w}^*\Vert_2^2\!\leq\!(1\!-\!\eta\mu) \Vert\bar{\mathbf{w}}_{t-1}\!-\!\mathbf{w}^*\Vert_2^2\!+\!\eta^2 B; \label{eq:delta rE+i}\\
    \forall t\in\mathcal{I}_E:\, &
    \mathbb{E}\!\left(\Vert\bar{\mathbf{w}}_t\!-\!\mathbf{w}^*\Vert_2^2\right)\!\leq\!(1\!+\!\sqrt{p_{\max}})\times \nonumber\\
    &\quad\quad\quad\quad \left[\!(1\!-\!\eta\mu) \!\Vert\bar{\mathbf{w}}_{t\!-\!1}\!-\!\mathbf{w}^*\!\Vert_2^2\!+\!\eta^2 \!B\!+\!\frac{X^{\rm BF}_t}{\!{\!\sqrt{p_{\max}}}\!}\!\right]\!,\label{eq:delta rE}
\end{align}
\end{subequations}
where~\eqref{eq:delta rE+i} is obtained by
plugging $\bar{\boldsymbol{\omega}}_t=\bar{\mathbf{w}}_t,t\notin\mathcal{I}_E$ in~\eqref{eq:delta_omega}.
As for~\eqref{eq:delta rE}, $\bar{\mathbf{w}}_{t}=\bar{\boldsymbol{\omega}}_{t}+\boldsymbol{z}_{t,\mathcal{G}}$ when $t=kE\in\mathcal{I}_E$; thus, 
 \begin{align}
   \notag   \mathbb{E}&\left(\Vert\bar{\mathbf{w}}_t\!\!-\!\!\mathbf{w}^*\Vert_2^2\right)=\mathbb{E}\left(\Vert\bar{\boldsymbol{\omega}}_t+\boldsymbol{z}_{t,\mathcal{G}}-\mathbf{w}^*\Vert_2^2\right)
   \\&\!\!\leq \!\!(1\!\!+\!\!\sqrt{p_{\max}}) \Vert\bar{\boldsymbol{\omega}}_t\!\!-\!\!\mathbf{w}^*\Vert_2^2\!\!+\!\!(1\!\!+\!\!\frac{1}{\sqrt{p_{\max}}})\mathbb{E}\left(\Vert\boldsymbol{z}_{t,\mathcal{G}}\Vert_2^2\right),\label{eq:delta rE2}
 \end{align}
where the expectation is taken to capture the randomness of bit-flipping noise in the $k$-th communication round. 
By substituting~\eqref{eq:Ez^2} and~\eqref{eq:delta_omega} in~\eqref{eq:delta rE2}, we obtain~\eqref{eq:delta rE}.

By substituting~\eqref{eq:delta rE+i} for $t\!= \!kE\!+\!i,\cdots,kE\!+\!1$ into $ \Vert\bar{\mathbf{w}}_{kE\!+\!i}\!-\!\mathbf{w}^*\Vert_2^2$ and using mathematical induction, it follows that 
\begin{align}
\!\!\!\!\!\!\Vert\bar{\mathbf{w}}_{kE\!+\!i}\!\!-\!\!\mathbf{w}^*\Vert_2^2\!\!\leq\!\!(1\!\!-\!\!\eta\mu)^{i}\Vert\bar{\mathbf{w}}_{kE}\!\!-\!\!\mathbf{w}^*\Vert_2^2\!\!+\!\!\eta^{2}B\frac{1\!\!-\!\!(1\!\!-\!\!\eta\mu)^{i}}{\eta\mu}.\label{eq:conver_kEi}
\end{align}
Then, we can obtain the upper bound of $\mathbb{E}\left(\Vert\bar{\mathbf{w}}_{T}-\mathbf{w}^*\Vert_2^2 \right)$ as 
\begin{subequations}\label{eq:conver_kE}\footnotesize
    \begin{align}
&\!\!\!\!\! \mathbb{E}\big(\Vert\bar{\mathbf{w}}_{T}\!\!-\!\!\mathbf{w}^*\Vert_2^2 \big)\!\!
\leq\!\!(1\!\!+\!\!\sqrt{p_{\max}})\!\left[\!(1\!\!-\!\!\eta\mu)\mathbb{E}\big(\Vert\bar{\mathbf{w}}_{T\!-\!1}\!\!-\!\!\mathbf{w}^*\Vert_2^2 \big)\!\!+\!\!\eta^{2}\!B\!\!+\!\!\frac{X^{\rm BF}_{T}}{{\sqrt{p_{\max}}}} \right]\label{eq:conver_kEa}\\&\!\!\!\!\!
 \leq\!\!(1\!\!+\!\!\sqrt{p_{\max}})\!\!\left[\!(1\!\!-\!\!\eta\mu)^{E}\mathbb{E}\!\left(\Vert\bar{\mathbf{w}}_{T\!-\!E}\!\!-\!\!\mathbf{w}^*\Vert_2^2\right)\!\!+\!\!\eta^{2}\!B\frac{1\!\!-\!\!(1\!\!-\!\!\eta\mu)^{E}}{\eta\mu}\!\!+\!\!\frac{ X^{\rm BF}_{T}}{\sqrt{p_{\max}}}\right]\label{eq:conver_kEb}\\\notag
&\!\!\!\!\!\leq(1+\sqrt{p_{\max}})^{K}(1\!\!-\!\!\eta\mu)^{T}\Vert\bar{\mathbf{w}}_{0}\!\!-\!\!\mathbf{w}^*\Vert_2^2\\&\!\!\!\!
\;+\!\!\!\!\underset{\forall i\leq K}{\sum}\!\!\!(1\!\!-\!\!\eta\mu)^{(i\!-\!1)E\!}(1\!\!+\!\!\sqrt{p_{\max}})^{i}(\eta^{2}\!B\frac{1\!\!-\!\!(1\!\!-\!\!\eta\mu)^{E}}{\eta\mu}\!\!+\!\!\frac{ X^{\rm BF}_{\!iE\!}}{\sqrt{p_{\max}}})\label{eq:conver_kEc}
\\\notag \leq&\!\!\frac{1\!\!-\!\!\left((1\!\!-\!\!\eta\mu)^{E}(1\!\!+\!\!\sqrt{p_{\max}})\right)^{K}}{1\!\!-\!\!(1\!\!-\!\!\eta\mu)^{E}(1\!\!+\!\!\sqrt{p_{\max}})}(1\!\!+\!\!\sqrt{p_{\max}})\Big(\eta^{2}B\frac{1\!\!-\!\!(1\!\!-\!\!\eta\mu)^{E}}{\eta\mu}\!\!+\!\!\frac{X^{\rm BF}}{\sqrt{p_{\max}}}\Big)
     \\&+(1+\sqrt{p_{\max}})^{K}(1-\eta\mu)^{T}\Vert\bar{\mathbf{w}}_{0}\!\!-\!\!\mathbf{w}^*\Vert_2^2,\label{eq:conver_kEd}
 \end{align}
\end{subequations}
where
\eqref{eq:conver_kEa} is based on~\eqref{eq:delta rE}.~\eqref{eq:conver_kEb} is obtained by substituting~\eqref{eq:conver_kEi} with $t=T-1$ into~\eqref{eq:conver_kEa}. 
Iterating~\eqref{eq:conver_kEb} for $t=K E,(K -1)E,\cdots,E$ yields~\eqref{eq:conver_kEc}. 
Eventually,~\eqref{eq:conver_kEd} is based on the definition $X^{\rm BF}=\underset{\forall t}{\max}\;X^{\rm BF}_t$.

\bibliographystyle{IEEEtran}
	\bibliography{ciations}
     \begin{IEEEbiography}[{\includegraphics[width=1in,height=1.25in,clip,keepaspectratio]{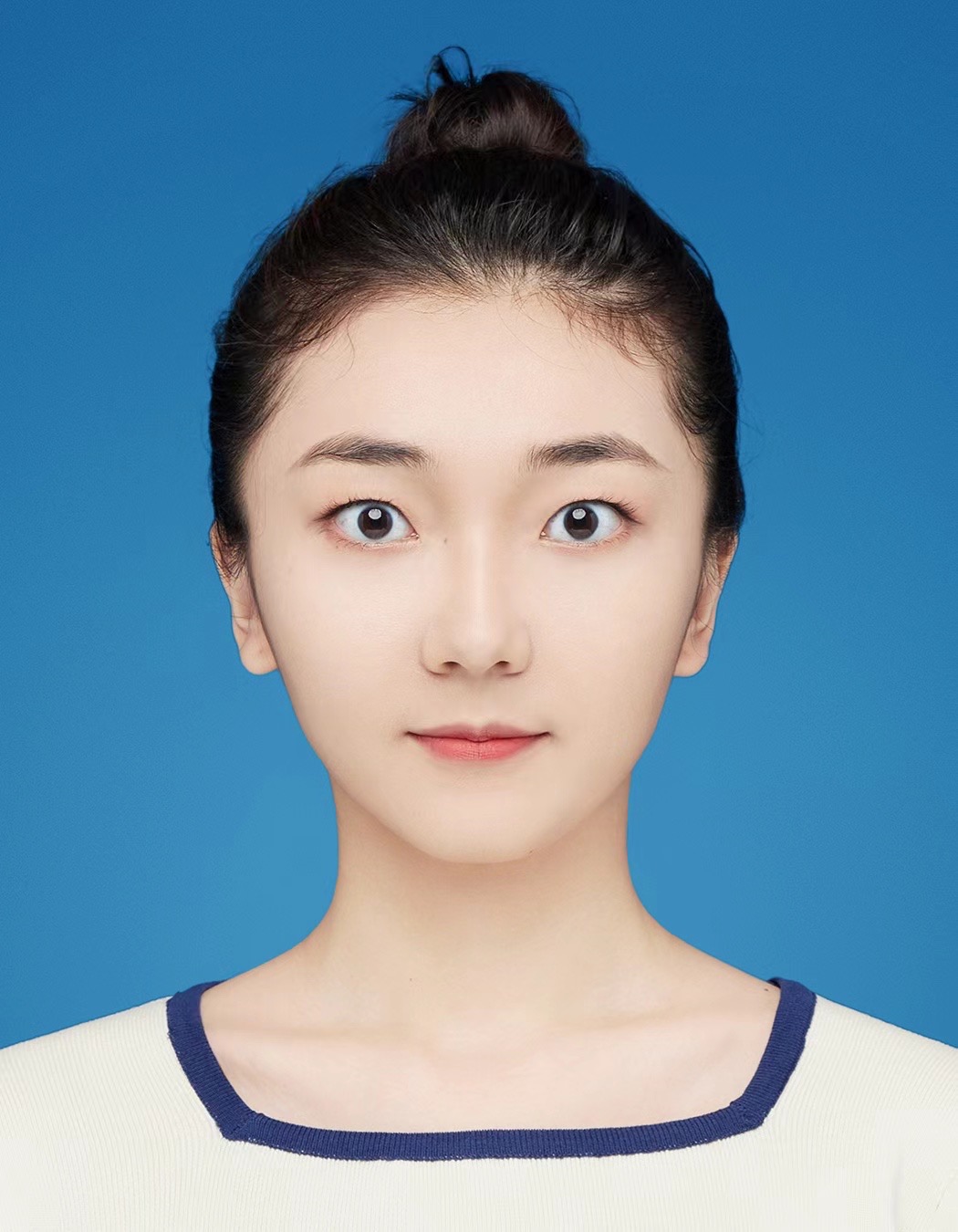}}]{Weicai Li} (Graduate Student Member, IEEE)
received the B.E. degree in communication engineering from Beijing University of Posts and Telecommunications (BUPT), China, in 2020. She is pursuing her Ph.D. with the School of Information and Communication Engineering at BUPT. From December 2022 to December 2023, she was a Visiting Scholar at the University of Technology Sydney. Her research interests include wireless federated learning, privacy computing, and semantic communication.
\end{IEEEbiography}
\begin{IEEEbiography}[{\includegraphics[width=1in,height=1.25in,clip,keepaspectratio]{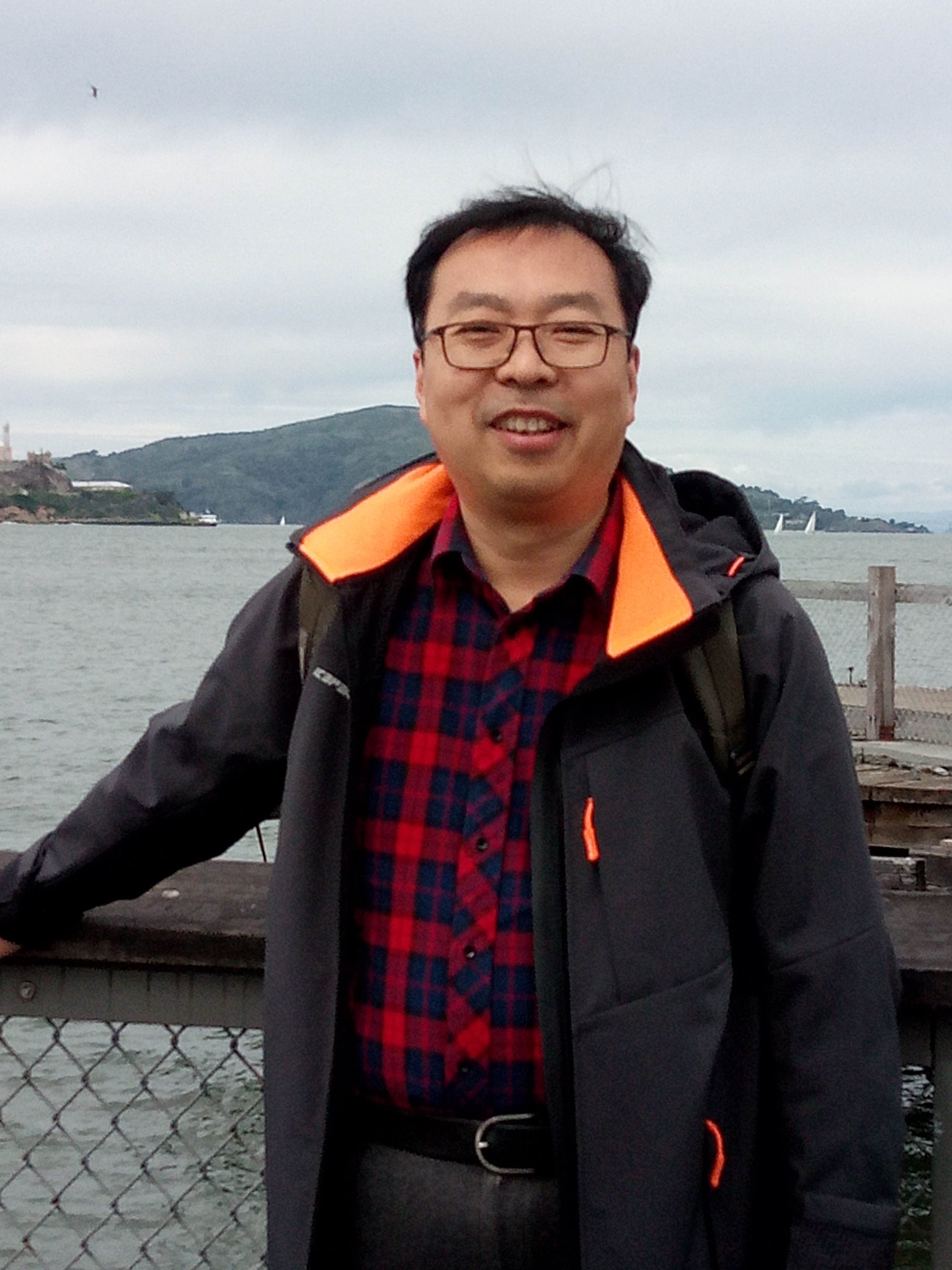}}]{Tiejun Lv}(Senior Member, IEEE) received the M.S. and Ph.D. degrees in electronic engineering from the University of Electronic Science and Technology of China (UESTC), Chengdu, China, in 1997 and 2000, respectively. From January 2001 to January 2003, he was a Postdoctoral Fellow at Tsinghua University, Beijing, China. In 2005, he was promoted to Full Professor at the School of Information and Communication Engineering, Beijing University of Posts and Telecommunications (BUPT). From September 2008 to March 2009, he was a Visiting Professor with the Department of Electrical Engineering at Stanford University, Stanford, CA, USA. He is the author of four books, one book chapter, more than 160 published journal papers and 200 conference papers on the physical layer of wireless mobile communications. His current research interests include signal processing, communications theory and networking. He was the recipient of the Program for New Century Excellent Talents in University Award from the Ministry of Education, China, in 2006. He received the Nature Science Award from the Ministry of Education of China for the hierarchical cooperative communication theory and technologies in 2015. He has won best paper award at CSPS 2022.
Dr. Lv has been an Editor for Sensors since 2024.
\end{IEEEbiography}
\begin{IEEEbiography}[{\includegraphics[width=1in,height=1.25in,clip,keepaspectratio]{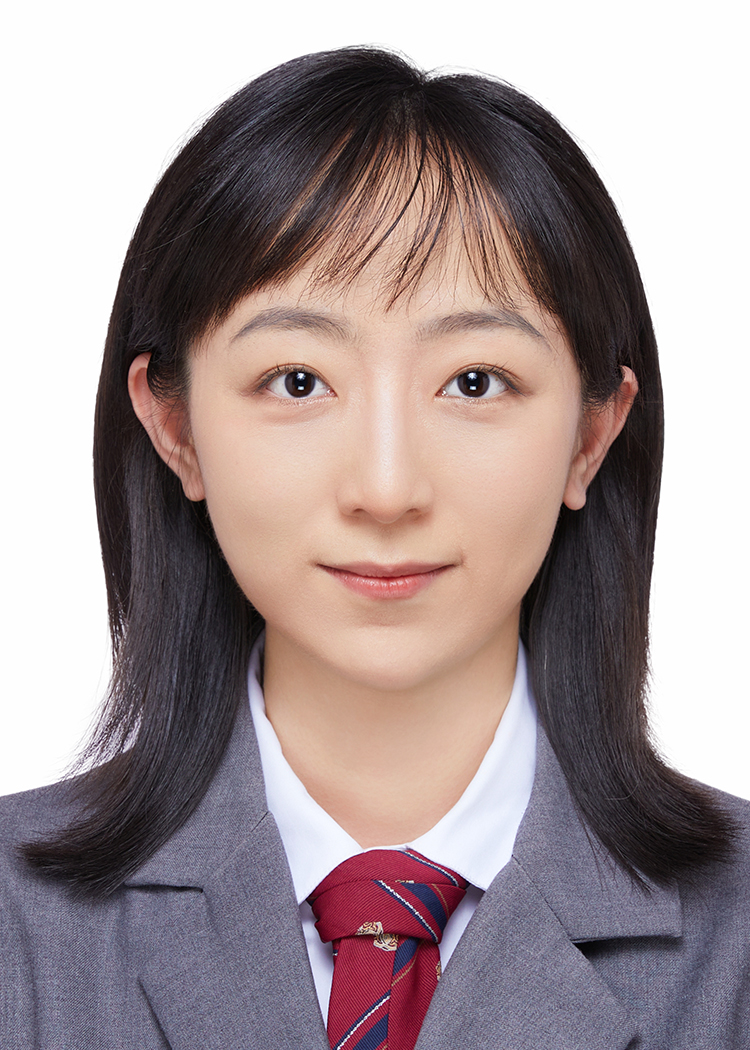}}]
{Xiyu Zhao}~received the B.E. degree in communication engineering from Beijing University of Posts and Telecommunications (BUPT), China, in 2020. She is pursuing her Ph.D. with the School of Information and Communication Engineering at BUPT. From June 2023 to December 2024, she was a Visiting Scholar at Macquarie University. Her research interests include wireless federated learning, distributed computing, and privacy-preserving.
\end{IEEEbiography}

\begin{IEEEbiography}
[{\includegraphics[width=1in,height=1.25in,clip,keepaspectratio]{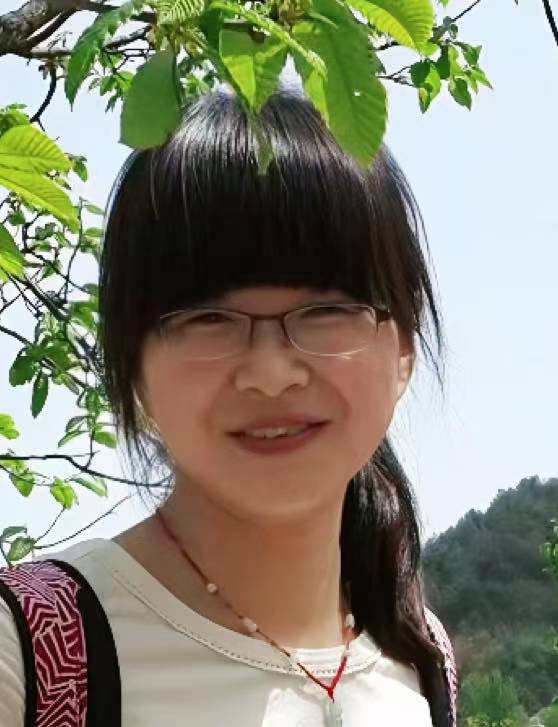}}]
{Xin Yuan} received the Ph.D. degree in communication engineering from Beijing University of Posts and Telecommunications (BUPT), China, in 2019.
Her research interests include wireless federated learning, distributed computing, and privacy-preserving.
\end{IEEEbiography}

\begin{IEEEbiography}
[{\includegraphics[width=1in,height=1.25in,clip,keepaspectratio]{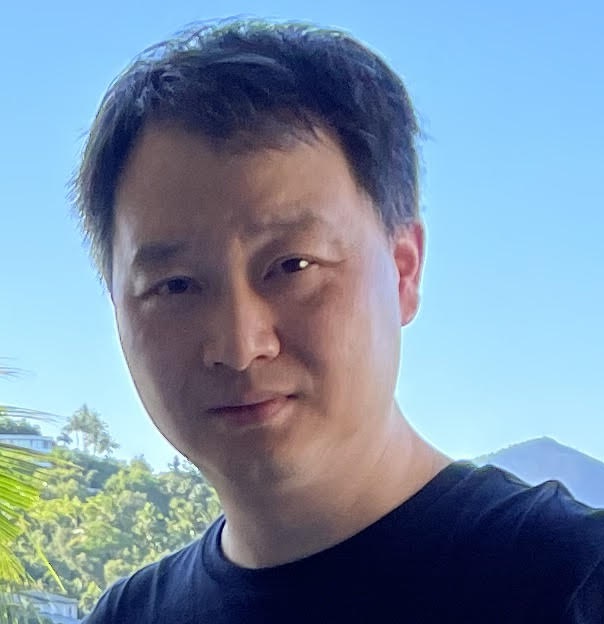}}]
{Wei Ni}~is a Mechatronics Program Advisor and Professor (Honorary) at Macquarie University, Sydney, New South Wales, Australia. His research interests include channel modeling, communications and signal processing, array signal processing, networking, multiple access, modulation and coding, radio access network (RAN), resource allocation, smart grid, blockchain, Internet-of-Things (IoT), autonomous systems, cyber-physical systems (CPS), and connected intelligent systems in general.
\end{IEEEbiography}

\end{document}